\documentclass[american]{llncs} %12pt
\usepackage{todonotes}
\usepackage{babel}
\usetikzlibrary{shapes,shapes.geometric,backgrounds,calc,fit,positioning}

\usepackage[utf8]{inputenc}
\usepackage{graphicx}

\usepackage[ruled,vlined]{algorithm2e}

\usepackage{amsmath,amssymb}
\usepackage{mathtools}
\usepackage{commath}
\usepackage{faktor} %nicefrac
\usepackage{ntheorem}
\usepackage{nicefrac}%nicer inline frak
\usepackage{enumerate}% enumerate
\usepackage{paralist}% inparaenum
\usepackage{mathdesign}

\usepackage[hidelinks]{hyperref}
\usepackage[all]{hypcap}

\usepackage[noabbrev]{cleveref}
\let\cref\Cref

\usepackage[backend=biber,style=numeric-comp,giveninits=true,url=false]{biblatex}
\usepackage{booktabs}
\usepackage{arydshln}%dashed hline

\title{Conceptual Views on Tree Ensemble Classifiers}

\date{tba 2022}

\institute{University of Kassel, Germany}

\usepackage{fca, newdrawline}

\crefalias{problem}{prob}

\newcommand{\Scon}{\mathbb{S}}

\DeclareMathOperator{\Ext}{Ext}
\DeclareMathOperator{\Int}{Int}

\usepackage{etoolbox,refcount}
\usepackage{multicol}
\newcounter{countitems}
\newcounter{nextitemizecount}
\newcommand{\setupcountitems}{%
  \stepcounter{nextitemizecount}%
  \setcounter{countitems}{0}%
  \preto\item{\stepcounter{countitems}}%
}
\makeatletter
\newcommand{\computecountitems}{%
  \edef\@currentlabel{\number\c@countitems}%
  \label{countitems@\number\numexpr\value{nextitemizecount}-1\relax}%
}
\newcommand{\nextitemizecount}{%
  \getrefnumber{countitems@\number\c@nextitemizecount}%
}
\newcommand{\previtemizecount}{%
  \getrefnumber{countitems@\number\numexpr\value{nextitemizecount}-1\relax}%
}
\makeatother    
    {\end{itemize}%
    \unskip\computecountitems\ifnumcomp{\previtemizecount}{>}{3}{\end{multicols}}{}}
\newcommand\blfootnote[1]{%
  \begingroup
  \renewcommand\thefootnote{}\footnote{#1}%
  \addtocounter{footnote}{-1}%
  \endgroup
}

\begin{document}
\parindent0pt

\author{Tom Hanika\inst{1,2} \and Johannes Hirth\inst{1,2}}

\date{\today}

\institute{%
  Knowledge \& Data Engineering Group,
  University of Kassel, Germany\\[0.5ex]
  \and
  Interdisciplinary Research Center for Information System Design\\
  University of Kassel, Germany\\[0.5ex]
  \email{tom.hanika@cs.uni-kassel.de, hirth@cs.uni-kassel.de}
}

\maketitle
\blfootnote{Authors are given in alphabetical order.
  No priority in authorship is implied.}

\begin{abstract}
  Random Forests and related tree-based methods are popular for
  supervised learning from table based data. Apart from their ease of
  parallelization, their classification performance is also
  superior. However, this performance, especially parallelizability,
  is offset by the loss of explainability. Statistical methods are
  often used to compensate for this disadvantage. Yet, their ability
  for local explanations, and in particular for global explanations,
  is limited. In the present work we propose an algebraic method,
  rooted in lattice theory, for the (global) explanation of tree
  ensembles. In detail, we introduce two novel conceptual views on
  tree ensemble classifiers and demonstrate their explanatory
  capabilities on Random Forests that were trained with standard
  parameters.
\end{abstract}

\section{Introduction}

Decision trees are among the most popular explainable machine learning
models. That is why they are often used as surrogates for other, less
transparent machine learning models.  However, one drawback is that
decision trees often do not perform as well as more contemporary
classification procedures. Furthermore, decision trees can naturally
not cope with missing data, i.e., without further help or data
preparation.

A popular class of classifiers, tree ensembles, do remedy these
disadvantages. They employ multiple tree structures simultaneously
(\emph{boosting}) and make potentially use of individual
\emph{baggings} of the data, e.g., Random Forests or Gradient Boosted
Trees~\cite{rforest,TBoost}.  While these methods are capable of high
classification performance, they do not possess the same
inter\-pre\-ta\-bi\-li\-ty as decision trees. This fact is due to the
dispersion of information into a large number of incomparable parallel
branches from differently rooted trees. There are quite attempts to
rectify this problem for explainability. One approach to explain tree
ensembles is merging all trees into a single decision
tree~\cite{Tmerging}. Even though the resulting tree structure can be
called (more) ``explainable'' (because it is a tree), it tends
to grow incomprehensibly large and also loses the ability to handle
missing values.

With the present work we propose a novel method for translating tree
ensembles into a data structure that is interpretable by design, while
allowing for parallelism to cope with missing information. We use
``lattices'' (i.e., ordered sets with supremum and infimum), in
contrast to trees~\cite{pmlr-v119-vidal20a}. While in a tree there is
always a unique path from the root to a node, lattices allow for
multiple (parallel) paths in a single lattice, leading to the same
conclusion element. This results in a better handling of missing
values.

Yet, we should not expect any data structure representing an ensemble
classifier to be handy and of manageable size. Thus, the resulting
lattices are large.  Fortunately, there is a mathematically sound
theory for lattices, which allows to generate manageable projections
of complex structures, which together contain the complete
information. We employ Formal~Concept~Analysis
(FCA)~\cite{Wille1982,fca-book}, which is able to deal with such
structures by design and offers a large framework of methods to
achieve representations of human comprehensible size. More
importantly, FCA provides a large set of tools to interpret these
representations and post-process them further.

Our work is not the first to take the step from tree ensembles to
lattices~\cite{dudyrev}, however, it drives this research in two
aspects: first, the proposed lattice representations achieve an
unprecedented expressiveness and thus explainability of tree ensemble
classifiers. Second, we provide a formal interpretation framework on
how to understand a classifier through the lens of conceptual
(lattice) views. These views allow for both local and global
explanations with varying levels of detail, as we will demonstrate.

In our experimental study, we demonstrate our explanation method and
its applicability on a real world example from the openml
CC18~\cite{openml} classification benchmark data set, namely,
\emph{car}~\cite{uci} (binary class
version\footnote{\url{https://www.openml.org/d/991}}). The analyzed
Random Forest was trained with realistic parameter assumptions, i.e.,
we used 100 trees and did not limit their individual depths.

\section{Motivation and Formal definitions}
\label{sec:formalities}
\FCA (FCA) is a mathematical theory of concepts and concept
hierarchies, with applications to Data Science and Knowledge
Processing. Since the purpose of this paper is to explore whether \FCA
can provide informative views of tree based classifiers, we want to
formally introduce relevant notions and notations.

\subsection{Formal concepts and conceptual views}
\label{sec:form-conc-conc}
At first glance, it seems very limiting that FCA focuses on one basic data
type, that of a binary relation between two sets (a \textbf{formal
  context} in the language of FCA), because data comes in many
different formats. But this limitation is intentional. It allows a
cleaner separation of objective formal data analysis and subjective
interpretation, and it allows a unified, clear structure of the 
mathematical theory.

FCA handles the many different data formats in a two-step
process. First, the data is transformed into the standard form --
that is, into a formal context -- and in the second step that context is
analyzed conceptually. The first step, called \textbf{conceptual
  scaling},\footnote{The word ``scaling'' here refers to measurement
  theory, not whether algorithms can be applied to large data sets.} 
is understood as an act of interpretation and depends on
subjective decisions of the analyst, who must reveal how the data at
hand is meant. It is therefore neither unambiguous nor automatic, and
usually does not map the data in its full complexity. However, it is
quite possible that several such \textbf{conceptual views} together
completely reflect the data. 

The classification data for tree classifiers are usually lists of
n-tuples, as in relational databases or in so-called data tables. In 
FCA terminology, ond speaks of a \textbf{many-valued context}. In such
a many-valued context, the rows have different names (thereby forming
a key), and so do the columns, whose names are called
\textbf{(many-valued) attributes}. The entries in the table are the
values. For classification it is usually assumed that the values of each 
attribute are linearly ordered. This order is then used to split the
data. This does not mean that the values have to be numeric. \texttt{low}
$<$ \texttt{medium} $<$ \texttt{high} and \texttt{true} $<$
\texttt{false} are examples of such \textbf{interordinal} data as
well. For numeric data one could, in addition to $a < b$, for
example use $a^2 < \cos(b)$ as a splitting criterion. Not doing so is
an interpretive decision in the sense discussed above.

To illustrate the notations that now follow, a small example of a
decision tree for the \emph{tennis play}~\cite{tennis} data set is
presented in  \cref{fig:tennis-dt}.

% For categorically interpreted attributes these predicates are checks
% for attribute values, e.g. $\emph{windy}=\emph{True}$ \textcolor{blue}{or \emph{windy?} in short}\bgfootnote{schief}
% for boolean attributes. For ordinally interpreted attributes these
% predicates test for a given relation to a value, e.g. for the ordinal
% temperature attribute $\emph{cool}\leq\emph{mild}\leq\emph{hot}$ the
% predicate $\emph{temperature}\leq\emph{mild}$ is satisfied by data
% objects that have temperature \emph{mild} or \emph{cool}. For
% classification, an object is threaded through the tree continuing with
% the child which's predicate the objects satisfies. When an object
% reaches a leaf node it is classified by an annotated label,
% e.g. object $0$ is classified by the leaf left in \cref{fig:tennis-dt}
% \emph{\color{red} ``no''} because it does not satisfy
% $\emph{humidity}\leq\emph{normal}$ and $\emph{overlook}\leq\emph{overcast}$.

% Throughout the rest of this work, we assume that every leaf node in
% the tree is supported by at least one object of the training data set
% and that there is at least one split.\bgfootnote{Oder weglassen, oder
%   später, ``Throughout''} These settings do no occur in practice 
% since they are not meaningful for the classification.

\begin{figure}
  \newcommand{\scale}{.75}%1
  \newcommand{\scaledt}{.54}%1
  \newcommand{\mpl}{.45}%1
  \newcommand{\mpr}{.54}%1
  \begin{center}
    \scalebox{\scale}{\begin{tabular}{|c||c|c|c|c||c|} \bottomrule
                        $\mathbb{D}$& overlook & temperature & humidity & windy & play\\
                        \hline \hline
                        0&sunny&hot&high&False&no\\
                        1&sunny&hot&high&True&no\\
                        2&overcast&hot&high&False&yes\\
                        3&rainy&mild&high&False&yes\\
                        4&rainy&cool&normal&False&yes\\
                        5&rainy&cool&normal&True&no\\
                        6&overcast&cool&normal&True&yes\\
                        7&sunny&mild&high&False&no\\
                        8&sunny&cool&normal&False&yes\\
                        9&rainy&mild&normal&False&yes\\
                        10&sunny&mild&normal&True&yes\\
                        11&overcast&mild&high&True&yes\\
                        12&overcast&hot&normal&False&yes\\
                        13&rainy&mild&high&True&no\\ \hline
                      \end{tabular}}
                  \end{center}
\vspace*{5mm} 
\begin{tikzpicture}[scale=\scaledt,sibling distance=10em, 
  level 1/.style={sibling distance=110mm},
  level 2/.style={sibling distance=53mm},
  level 3/.style={sibling distance=52mm},
  level 4/.style={sibling distance=30mm},
  output/.style={level distance=6ex},
  output-node/.style={text=red!80!white,draw=none},
  every node/.style = {shape=rectangle, rounded
    corners, draw, align=center}]
  \node[shape=circle] {}
  child { node {$\text{humidity}\leq\text{normal}$}
          child {node {$\text{overlook}\leq\text{rainy}$}
                 child {node {$\text{not windy}$}
                        child[output] {node[output-node] {yes}edge from parent[draw=none]}}
                 child {node {\text{windy}}
                        child[output] {node[output-node] {no}edge from parent[draw=none]}}}
          child {node {$\text{overlook}\geq\text{overcast}$}
                 child[output] {node[output-node] {yes}edge from parent[draw=none]}}} 
  child { node {$\text{humidity}\geq\text{high}$}
          child {node {$\text{overlook}\leq\text{overcast}$}
                 child {node {$\text{overlook}\leq\text{rainy}$}
                        child {node {$\text{not windy}$}
                               child[output] {node[output-node] {yes} edge from parent[draw=none]}}
                        child {node {windy}
                               child[output] {node[output-node] {no} edge from parent[draw=none]}}}
                 child {node {$\text{overlook}\geq\text{overcast}$}
                        child[output] {node[output-node] {yes} edge from parent[draw=none]}}}
          child {node {$\text{overlook}\geq\text{sunny}$}
                 child[output] {node[output-node] {no} edge from parent[draw=none]}}};
\end{tikzpicture}

\caption{Decision tree for the tennis data set. Each data object
  follows the path from the root (the top node) along the predicates it
  fulfills until it reaches a leaf node. The decision for ``play'' highlighted
  annotated in red.}  \label{fig:tennis-dt}
\end{figure}
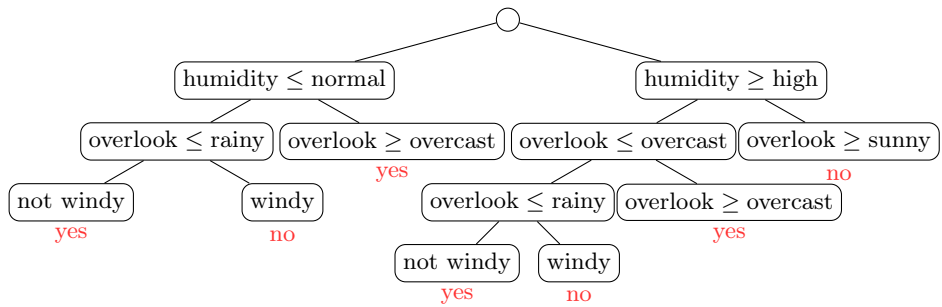

\subsection{Tree based  classifiers}
\label{sec:treebased}
As usual a decison tree $\mathcal{T}$ is a proper binary tree, i.e., a
rooted tree in which every non-terminal node\footnote{Terminal nodes
  are also called \emph{leaves}, non-terminal ones are \emph{inner}}
has exactly two children. With the exception of the root $r$, each node
$n$ is annotated by a \emph{predicate} $\phi(n)$ that \emph{splits}
the data, so that every data object either 
satisfies $\phi(n)$ or its negation $\neg\phi(n)$. The annotation is
such that the two child nodes of an inner node are negated to each
other: if one is annotated by $P$, then the other is
annotated by $\neg P$. According to these definitions the set
$\mathcal{P}(\mathcal{T})\coloneqq \{\phi(n)\mid n\in
\mathcal{T}\setminus\{r\}\}$ of predicates used for $\mathcal{T}$
  contains the negation $\neg P$ of every $P\in \mathcal{P}(\mathcal{T})$.

Each such tree
carries a natural order, in which for nodes $m,n$ we have $$m\le
n:\iff n \mbox{ lies on the unique path from }m\mbox{ to the root.}$$

The example in \cref{fig:tennis-dt} shows in its upper part a data
table with fourteen data objects, numbered $0, \ldots, 13$. There are
four  attributes, ``overlook'', ``temperature'', ``humidity'',
and ``windy'', while the last column ``play'' contains the
classification outcome. 
Each attribute is used with an
ordinal interpretation of its values, such as
\begin{center}\itshape
  rainy $<$ overcast $<$ sunny,\\
  cool $<$ mild $<$ hot,\\
  normal $<$ high, \mbox{\rm and} \\
  not windy $<$ windy.
\end{center}

In addition to that we derive an \emph{interordinal} interpretation of
these values using the $\geq$ and $\leq$ relations.  For
classification, a data object $g$ is threaded through the tree from
the root to a leaf node $b$ such that it satisfies all predicates
along the path $r,\dots,b$. We call $b$ the
\emph{classification/decision leaf} and the path $r,\dots,b$ the
\emph{decision path} of $g$. Finally there is a second mapping,
associating to every leaf the classification outcome, usually ``yes''
or ``no''.

\subsection{Conceptual views via conceptual scaling}
\label{sec:con-view-scaling}
A \textbf{formal context} is a triple $\context:=(G,M,I)$, where $G$
and $M$ are sets and $I\subseteq G\times M$ is a binary relation. As
said above, this is the basic data structure used in \FCA. The
elements of $G$ are called the \textbf{objects}, those of $M$ the
\textbf{attributes} of the formal context $\context$. One reads
$(g,m)\in I$ as \emph{``{}the object $g$ has the attribute $m$''.}

A \textbf{many-valued context} $\mathbb{D}\coloneqq (G,M,W,I)$ again has a set $G$ of
\emph{objects}, a set $M$ of \textbf{many-valued attributes}, and a
set $W$ of \textbf{attribute values}. The relation $I\subseteq G\times
M\times W$ here is ternary, with $(g,m,w)\in I$ expressing that
\emph{the object $g$ has the value $w$ for the attribute $m$.} It is
assumed that each object has at most one value for each attribute, so
that $$(g,m,w_{1})\in I, (g,m,w_{2})\in I \implies w_{1}=w_{2}.$$ 

Furthermore, we call $\mathbb{D}$ \emph{complete} iff for every
$g\in G,m\in M$ there is a $w\in W$ with $(g,m,w)\in I$.  By
$m^{\mathbb{D}} \coloneqq \{w\in W \mid\exists g\in G (g,m,w)\in I\}$
we denote the set of all values $w\in W$ such that there is an object
$g\in G$ that is in incidence with $w$ for attribute $m$.  The process
of deriving a formal context from a many-valued one is called
\textbf{conceptual scaling}. There are several variants for this. The
simplest is \textbf{plain scaling}, where one simply specifies how the
sets of values of each many-valued attribute are to be understood. For
this purpose, a formal context (a ``\textbf{scale}'') is specified for
each many-valued attribute, such that the set of objects of this scale
contains the values of this attribute. Each value then is replaced by
the corresponding row of that scale.  \cref{fig:scaling} demonstrates
this for the tennis data set.  The many-valued context here is given
by the data table in \cref{fig:tennis-dt} (ignoring the last
column). For each of the four many-valued attributes a scale is given
in the left half of \cref{fig:scaling}. The formal context on the
right was derived from the data table in \cref{fig:tennis-dt} via
plain scaling, i.e., by replacing the values by the respective rows of
the scales.

\begin{figure}[t]
  \centering
  \newcommand{\ro}{90}
  \newcommand{\ros}{90}
  \newcommand{\row}{77}
  \newcommand{\roattr}{45}
  \newcommand{\scale}{.8}%1
  \newcommand{\mpl}{.5}%1
  \newcommand{\mpr}{.49}%1
  \begin{minipage}{\mpl\linewidth}
  \scalebox{\scale}{\begin{tabular}{|c||c|c|c|c|}
    \bottomrule
    \rotatebox{\roattr}{\phantom{x.}Overlook\phantom{xx}} 
    & \rotatebox{\ro}{$\leq \emph{rainy}$}
    & \rotatebox{\ro}{$\leq \emph{overcast}$} 
    & \rotatebox{\ro}{$\geq \emph{overcast}$} 
    & \rotatebox{\ro}{$\geq \emph{sunny}$}  \\
    \hline \hline
    rainy&$\times$&$\times$&& \\
    overcast&&$\times$&$\times$&\\
    sunny&&&$\times$&$\times$\\
 \hline
  \end{tabular}}
\scalebox{\scale}{\begin{tabular}{|c||c|c|}
    \bottomrule
   \rotatebox{\roattr}{Humidity} 
                    & \rotatebox{\ro}{$\leq \emph{normal}$} 
                    &\rotatebox{\ro}{$\geq \emph{high}$} \\
    \hline \hline
    normal&$\times$&\\
    hot&&$\times$\\ \hline
  \end{tabular}}\\
  \scalebox{\scale}{\begin{tabular}{|c||c|c|c|c|}
    \bottomrule
   \rotatebox{\roattr}{Temperature} 
                      & \rotatebox{\ro}{$\leq \emph{cold}$} 
                      & \rotatebox{\ro}{$\leq \emph{mild}$} 
                      & \rotatebox{\ro}{$\geq \emph{mild}$} 
                      & \rotatebox{\ro}{$\geq \emph{hot}$}  \\
    \hline \hline
    cold&$\times$&$\times$&&\\
    mild&&$\times$&$\times$&\\
    hot&&&$\times$&$\times$ \\ \hline
  \end{tabular}}
  \scalebox{\scale}{\begin{tabular}{|c||c|c|}
    \bottomrule
    \rotatebox{\roattr}{Wind} 
                      & \rotatebox{\ro}{\emph{windy}} 
                      & \rotatebox{\ro}{\emph{not windy}} \\
    \hline \hline
    True&$\times$&\\
    False&&$\times$\\ \hline
  \end{tabular}}
  \end{minipage}
  \begin{minipage}{\mpr\linewidth}
    \scalebox{\scale}{
\begin{tabular}{|c||c|c|c|c|c|c|c|c|c|c|c|c|}
\bottomrule
  $\mathbb{I}(\mathbb{D})$& \rotatebox{\ros}{$\emph{Overlook}\leq\emph{rainy}$}
  & \rotatebox{\ros}{$\emph{Overlook}\leq\emph{overcast}$}
  & \rotatebox{\ros}{$\emph{Overlook}\geq\emph{overcast}$}
  & \rotatebox{\ros}{$\emph{Overlook}\geq\emph{sunny}$}
  & \rotatebox{\ros}{$\emph{Temperature}\leq\emph{cool}$}
  & \rotatebox{\ros}{$\emph{Temperature}\leq\emph{mild}$}
  & \rotatebox{\ros}{$\emph{Temperature}\geq\emph{mild}$}
  & \rotatebox{\ros}{$\emph{Temperature}\geq\emph{hot}$}
  & \rotatebox{\ros}{$\emph{Humidity}\leq\emph{normal}$}
  & \rotatebox{\ros}{$\emph{Humidity}\geq\emph{high}$}
  & \rotatebox{\ros}{\emph{windy}}
  & \rotatebox{\ros}{\emph{not windy}}\\
    \hline \hline
0& &&$\times$&$\times$ & &&$\times$&$\times$ & &$\times$ & &$\times$\\ 
1& &&$\times$&$\times$ & &&$\times$&$\times$ & &$\times$ & $\times$&\\ 
2& &$\times$&$\times$& & &&$\times$&$\times$ & &$\times$ & &$\times$\\ 
3& $\times$&$\times$&& & &$\times$&$\times$& & &$\times$ & &$\times$\\ 
4& $\times$&$\times$&& & $\times$&$\times$&& & $\times$& & &$\times$\\ 
5& $\times$&$\times$&& & $\times$&$\times$&& & $\times$& & $\times$&\\ 
6& &$\times$&$\times$& & $\times$&$\times$&& & $\times$& & $\times$&\\ 
7& &&$\times$&$\times$ & &$\times$&$\times$& & &$\times$ & &$\times$\\ 
8& &&$\times$&$\times$ & $\times$&$\times$&& & $\times$& & &$\times$\\ 
9& $\times$&$\times$&& & &$\times$&$\times$& & $\times$& & &$\times$\\ 
10& &&$\times$&$\times$ & &$\times$&$\times$& & $\times$& & $\times$&\\ 
11& &$\times$&$\times$& & &$\times$&$\times$& & &$\times$ & $\times$&\\ 
12& &$\times$&$\times$& & &&$\times$&$\times$ & $\times$& & &$\times$\\ 
13& $\times$&$\times$&& & &$\times$&$\times$& & &$\times$ & $\times$&\\ 
\hline
  \end{tabular}}
  \end{minipage}
  \caption{The conceptual scaling of the tennis data set
    (\cref{fig:tennis-dt}). Each incidence is indicated by a $\times$
    in the cross-tables. The concept lattice of the derived context has 108
    concepts.} % 108 concepts
  \label{fig:scaling}
\end{figure}

An important step in using FCA is to unfold the formal context into a
\textbf{concept lattice}. The elements of this algebraic structure are called
\textbf{formal concepts}. They are defined as pairs $(A,B)$ of sets, where
$A\subseteq G$, $B\subseteq M$, $A'=B$ and $A=B'$ hold with 
$A' :=\{m\in M\mid (g,m)\in I \mbox{ for all }g\in A\}$ and $B':=\{g\in
G\mid (g,m)\in I \mbox{ for all } m\in B\}$. $A$ is called the
\textbf{extent} and $B$ the \textbf{intent} of the formal concept $(A,B)$.
The derived context in 
\cref{fig:scaling} has 108 formal concepts, which is considerably
larger than the size of the decision tree in \cref{fig:tennis-dt}.

This example already demonstrates that the full concept lattice of the
formal context derived from plain conceptual scaling is too complex to
be instructive. It indeed unfolds all of the conceptual structure of
the scaled data table. Instead it may be useful to study only
carefully selected subsets of the derived attributes and the
sub-contexts induced by these. This is what Wille calls conceptual
views in \textcite{conf/fca/Wille05}. In the following text, we refer by
\emph{contextual views} of a many-valued context $\mathbb{D}$ to
formal contexts that were derived from $\mathbb{D}$ by means of
conceptual scaling. Its corresponding concept lattice, or parts
thereof, are called \emph{conceptual views}. As the conceptual and the
contextual view are in one to one correspondence, we will often simply
refer to \emph{views}.

\textbf{Logical scaling}~\cite{logical-scaling} goes one step further
than plain scaling. With logical scaling it is allowed to use
(propositional) logical combinations of the scale attributes as
attributes of a derived context. These scale attributes may originate
from different many-valued attributes. This simple technique is easier
to understand with an example. In the tennis data set, the weather
conditions are described by specifying values for humidity, wind,
etc. You may want to use other attributes that can be composed from
the given ones, such as
\begin{center}\itshape
nice := (temperature = mild) $\wedge$ $\neg$windy.
\end{center}
Logical scaling allows this. Again, we will limit ourselves to a small
selection of such attributes. That is why we speak of a
\textbf{conceptual view} here as well.

\subsection{Learning a decision tree from a many valued context}
In order to train a decision tree on a data table (which is formalized
as a many-valued context $\mathbb{D}=(G,M,W,I)$), we need to know the
classification labels for all $g\in G$. The tennis data set in
\cref{fig:tennis-dt} shows an example. The last column contains the
classification labels, while the rest of the table represents a
many-valued context. However, this example is tiny compared to
realistic datasets from real life. For such one needs fast
implementations like the C4.5
algorithm~\cite{DBLP:books/mk/Quinlan93}. This algorithm chooses in
every step $n$ that \emph{predicate} $\phi(n)$ that most effectively
splits its set $G$ of samples into subsets. These predicates usually
are attribute-value pairs, and we write $g\models\phi(n)$ if the
object $g$ has this value for the attribute and $g\models\neg \phi(n)$
otherwise. The term \emph{most effectively} is taken with respect to
some measure of information entropy.

The results is a decision tree $\mathcal{T}$, as described
in~\cref{sec:treebased}. In the first step, the splitting decision is
made for the root node, and the set $G$ of all samples is divided into
two sets, which are passed to the child nodes. The predicate that was
used for this split is written as annotation to the child nodes, in
positive or negated form, respectively. Any node $n\in\mathcal{T}$,
with exception of the root, thereby receives an annotation (for
completeness, the root is annotated with $\top$), and a data object
$g\in G$ reaches a leaf of the tree if and only if it is a model of
each predicate annotating a node on the unique path to that leaf.   

For a given decision tree $\mathcal{T}$ and its training data set
$\mathbb{D}$, two main explanatory tasks can be formulated. The first
addresses the question how adequately $\mathcal{T}$ represents the
training data, in particular the objects $g\in G$, and which general
explanations can be inferred from $\mathcal{T}$ using $G$. These
explanations range from local ones, i.e., why was an object $g$
classified to a particular class from $\mathcal{C}$, and global ones,
such as, which predicate combination describe a class.  The second
task is to understand the \emph{view} of $\mathcal{T}$ on a so far
unknown set of objects $\check G$, of which every object $g\in\check
G$ is represented using the attributes and value domains of
$\mathbb{D}$, and may even have values missing. Again, these views can
be locally and globally.  
%% TODO Was ist mit dem Text hier? 
% scaled tennis data (cf. \cref{fig:scaling}) set has 108 concepts and is to large to
% visualize readably. The degree of detail encoded in these scales might
% not be needed, especially for 
% the classification. Thus, further down scaling of the
%  scale is suited. Determining such a data scaling that is
% important for the classification task is done by classification
% models, that are trained to identify certain patterns or attribute
% combinations to predict a target. In the next section, we discuss
% different scaling that can be extracted from a decision tree
% classifier.\todo{reference to dt not really clear}
 % positive attributes of tennis only is 31

\section{Concept Lattices from Tree Classifiers}\label{dtree-scales}
In the following we introduce different conceptual views on tree
classifiers. We first revisit three approachs from the literature and
derive a unified representation for them in the language of Formal
Concept Analysis.  We will consider these methods as baselines for our
two novel approaches in the next section.  

\subsection{Approaches from the literature}
\label{sec:appr-from-liter}

% Concept lattices may lead to a data oriented explanation
% method, i.e., our methods wants to capture the 
% \emph{view} of a tree classifier $\mathcal{T}$ on data
% (cf.~\cref{problem:treescale}). To this end, we will

The first approach to investigate is given by the
\texttt{RandomTreesEmbedding} from
sklearn\footnote{\url{https://scikit-learn.org/stable/modules/generated/sklearn.ensemble.RandomTreesEmbedding.html}}
and simply reflects the clustering of the data objects induced by the
leaf nodes of the decision tree
\cite{tree-clustering1,tree-clustering2,tree-clustering3}. 

% With the following definition we derive a formal
% representation of \texttt{RandomTreesEmbedding} within the language of
% formal contexts. For this we employ logical scaling which characterizes the decision
% paths leading towards the leafs $\mathcal{L}(\mathcal{T})$.

\begin{definition}[Conceptual Leaf View on $\mathcal{T}$]
  \label{def:leafscale}
  Given a many-valued context $\mathbb{D}\coloneqq (G,M,W,I)$ and a
  decision tree $\mathcal{T}$ that was trained on $\mathbb{D}$, we
  define the \emph{contextual leaf view on} $\mathcal{T}$ as
  \begin{equation*}
    \label{eq:2}
    \mathbb{N}(G,\mathcal{T})\coloneqq (G,\mathcal{L}(\mathcal{T}),J)
    \text{, where } (g,l)\in J \text{ iff } g\models\phi(n) \text{ for
      all } n\geq_{\mathcal{T}} l.
  \end{equation*}
  The corresponding concept lattice $\BV(\mathbb{N}(G,\mathcal{T}))$
  is called the \emph{conceptual leaf view on} $\mathcal{T}$. 
  A slight generalization allows to include not only the training
  data, but also other data objects with the same attributes and
  attribute values. $G$ may thus be replaced by $\check G$, as
  introduced at the end of Section~\ref{sec:formalities}. The set
  $\mathcal{L}(\mathcal{T})$, which is used as set of attributes here,
  is just the set of leaves of the decision tree.
\end{definition}

The incidence with such a leaf-attribute is defined by the conjunction
of all predicated on the unique path from the root to that leaf. This
being a logical combination of plain scaling attributes shows, that
this context is derived from \emph{logical scaling}. The particular
\emph{view}, i.e., which logical combinations were selected, depends
on the decision tree.

The conceptual leaf view using a particular set of objects $\check G$
enables to \emph{view} said objects through the classification leaves
of $\mathcal{T}$. This means, for any two objects
$g_{1},g_{2}\in\check G$ that are classified by the same leaf $l$, we
have that $\{g_{1}\}^{J}=\{g_{2}\}^{J}$, and therefore the object
concepts $(\{g_{1}\}^{JJ},\{g_{1}\}^{J})$ and
$(\{g_{2}\}^{JJ},\{g_{2}\}^{J})$ are equal. Informally said, $g_{1}$
and $g_{2}$ are clustered in the same concepts. We may note at this
point, that this view is limited to this fact. Hence, objects that are
classified by different leaf nodes do not share any attributes and are
therefore incomparable. \cref{fig:tennis-scales} shows the derived
context for our running example. 

More general, the conceptual leaf view is disappointingly simple, it
is just an antichain plus a top and a bottom element. This is due to
the fact that $\BV(\mathbb{N}(\check G,\mathcal{T}))$ is of nominal
scale, that is, it has exactly one concept per leaf, and all these
concepts are pairwise incomparable.  Altogether, this view is coarse
and does not exhibit hierarchical information, i.e, there are no
concepts in sub-concept relation, apart from those involving the top
$(\check G,\check G^{J})$ or bottom
$(\mathcal{L}(T)^{J},\mathcal{L}(T))$.

The second baseline view accounts for the whole order structure of the
decision tree $\mathcal{T}$~\cite{tree-clustering3}. The corresponding
concept lattice is an isomorphic representation to the one proposed by
previous work~\cite{dudyrev}.

\begin{definition}[Conceptual Tree View on $\mathcal{T}$]
  \label{def:treescale}
  Given a many-valued context $\mathbb{D}\coloneqq (G,M,W,I)$ and a
  decision tree $\mathcal{T}$ that was trained on $\mathbb{D}$, we
  define the \emph{contextual tree view on} $\mathcal{T}$ by taking
  the tree nodes as attributes. A node is incident with a data object
  if and only if it was used for classifying that object.
  \begin{equation*}
    \label{eq:3}
    \mathbb{T}(G, \mathcal{T})\coloneqq (G,\mathcal{T},J) \text{, where }
    (g,t)\in J \text{ iff } t \text{ is on the decision path of } g
    \text{ in }\mathcal{T}.
  \end{equation*}
  The corresponding concept lattice $\BV(\mathbb{T}(G, \mathcal{T}))$
  is called the \emph{conceptual tree view on} $\mathcal{T}$. 
  As in Definition~\ref{def:leafscale} above, $G$ may be replaced by a
  more general set $\check G$.
\end{definition}
In analogy to the leaf partition context, the contextual tree view can
also be understood as the result of a logical scaling. Each node is
then replaced by the conjunction of the predicates annotated along its 
path. 

In contrast to the contextual leaf view on $\mathcal{T}$ the
contextual tree view accounts for all nodes of $\mathcal{T}$.  Hence,
objects that are classified by different leaf nodes may have common
nodes in their decision paths. The more their respective decision
paths overlap, the more attributes they have in common.

The concept lattice $\BV(\mathbb{T}(G,\mathcal{T}))$ is
order-isomorphic to the decision tree with an added smallest element,
i.e., to $(\mathcal{T}\cup\{\bot \},\leq)$ where for all $n\in
\mathcal{T}:\ \bot \leq n$. The conceptual tree view on $\mathcal{T}$
can thus be considered as an almost one-to-one translation of the
decision tree into the realm of \FCA.

For a given arbitrary object sets $\check G$, only parts of the order
structure $(\mathcal{T}\cup\{\bot \},\leq)$ are reached. More precise,
since for all elements of $\check G$ there is an element of $G$ having
the same decision path, we can conclude that there is a (unique up to
context clarification) isomorphism from $\mathbb{T}(\check
G,\mathcal{T})$ into a sub-context of
$\mathbb{T}(G,\mathcal{T})$. Thus, $\Int(\mathbb{T}(\check
G,\mathcal{T}))\subseteq \Int(\mathbb{T}(G,\mathcal{T}))$. However,
for the rest of our work, we will not explore this relationship
further. Yet, we would like to refer the reader to our investigations
on concept measurements~\cite{smeasure}, in particular to Corollary
22.

\begin{definition}
  For a linearly ordered set $(V,\le)$ the (one-dimensional)
  \textbf{interordinal scale} over $V$ is the formal context
  $$(V,\{\leq,\geq\}\times V, \models),$$  where the incidence relation
  $\models$ is defined in the obvious manner: $$v\models (\leq,w):\iff
  v\leq w\mbox{ and }v\models(\geq,w):\iff v\geq w.$$ 
\end{definition}

\begin{figure}[t]
  \newcommand{\ro}{90}
  \newcommand{\x}{$\times$}
  \newcommand{\0}{\phantom{$\times$}}
  \newcommand{\scale}{.85}%1
\begin{minipage}{.5\linewidth}\centering
  \scalebox{\scale}{\begin{tabular}{|c||c|c|c|c|c|c|c|}
    \bottomrule
$\mathbb{N}(G,\mathcal{T})$    &$l_0$&$l_1$&$l_2$&$l_3$&$l_4$&$l_5$&$l_6$\\
    \hline \hline
    0&$\times$&&&&&&\\
    1&$\times$&&&&&&\\
    2&&$\times$&&&&&\\
    3&&&$\times$&&&&\\
    4&&&&&&$\times$&\\
    5&&&&&&&$\times$\\
    6&&&&&$\times$&&\\
    7&$\times$&&&&&&\\
    8&&&&&$\times$&&\\
    9&&&&&&$\times$&\\
    10&&&&&$\times$&&\\
    11&&$\times$&&&&&\\
    12&&&&&$\times$&&\\
    13&&&&$\times$&&&\\
    \hline
  \end{tabular}}
\end{minipage}
\begin{minipage}{.5\linewidth}\centering
  \scalebox{\scale}{
    \begin{tabular}{|c||c|c|c|c|c|c|c|c|c|c|c|c|}
      \bottomrule
      \rotatebox{0}{$\mathbb{P}(G,\mathcal{T})$}  &   
      \rotatebox{\ro}{$n_1$}&
      \rotatebox{\ro}{$n_2$}&
      \rotatebox{\ro}{$n_3$}&
      \rotatebox{\ro}{$n_4$}&
      \rotatebox{\ro}{$n_5$}&
      \rotatebox{\ro}{$n_6$}&
      \rotatebox{\ro}{$n_7$}&
      \rotatebox{\ro}{$n_8$}&
      \rotatebox{\ro}{$n_9$}&
      \rotatebox{\ro}{$n_{10}$}&
      \rotatebox{\ro}{$n_{11}$}&
      \rotatebox{\ro}{$n_{12}$}\\
      \hline \hline
       0&\0&\x&\0&\0&\0&\x&\0&\0&\0&\0&\0&\0\\
       1&\0&\x&\0&\0&\0&\x&\0&\0&\0&\0&\0&\0\\
       2&\0&\0&\0&\0&\x&\0&\0&\0&\0&\x&\0&\0\\
       3&\0&\x&\0&\0&\x&\0&\0&\0&\x&\0&\x&\0\\
       4&\x&\0&\x&\0&\0&\0&\x&\0&\0&\0&\0&\0\\
       5&\x&\0&\x&\0&\0&\0&\0&\x&\0&\0&\0&\0\\
       6&\x&\0&\0&\x&\0&\0&\0&\0&\0&\0&\0&\0\\
       7&\0&\x&\0&\0&\0&\x&\0&\0&\0&\0&\0&\0\\
       8&\x&\0&\0&\x&\0&\0&\0&\0&\0&\0&\0&\0\\
       9&\x&\0&\x&\0&\0&\0&\x&\0&\0&\0&\0&\0\\
       10&\x&\0&\0&\x&\0&\0&\0&\0&\0&\0&\0&\0\\
       11&\0&\0&\0&\0&\x&\0&\0&\0&\0&\x&\0&\0\\
       13&\0&\x&\0&\0&\x&\0&\0&\0&\x&\0&\0&\x\\
      \hline
    \end{tabular}}
\end{minipage}
\caption{The contextual leaf view on $\mathcal{T}$ (right in
  \cref{fig:tennis-dt}) for the running example (left in
  \cref{fig:tennis-dt}). Its concept lattice contains 55 concepts. The
  contextual tree view  $\mathbb{P}(G,\mathcal{T})$ of $\mathcal{T}$ is shown on the
  right.}
  \label{fig:tennis-scales}
\end{figure}

Data sets for decision tree classification have\footnote{If not so,
  they can be ordered linearly. } linearly ordered value sets and
therefore associated interordinal scales for all attributes. These can
be used for plain scaling. 

The third and last baseline that we want to
introduce is the most expressive in terms of formal concepts. In
contrast to the other two views, it solely depends on the many-valued
data set $\mathbb{D}$. Although its relevance to the scope of this
work may not be apparent at this point, the next section will
elaborate its importance.

\begin{definition}[Interordinal Scaling of $\mathbb{D}$]
  \label{def:interordinal}
  When $\mathbb{D}\coloneqq (G,M,W,I)$ is a many-valued context with
  linearly ordered value sets $(m^{\mathbb{D}},\leq_{m})$ for all many-valued 
  attribute sets $W(m)$, then the formal context \textbf{derived from
    interordinal scaling} has $G$ as its object set and attributes of
  the form $$(m,\leq_{m}, v)\mbox{ or }(m,\geq_{m}, v),$$ where $v$ is a value
  of the many-valued attribute $m$. The incidence is the obvious one,
  an object $g$ has e.g., the attribute $(m,\leq_{m},v)$ iff the value of
  $m$ for the object $g$ is $\leq_{m} v$. Instead of
  $(m,\leq_{m},v)$ or  $(m,\geq_{m},v)$ one  writes $m{\,:\;}\leq v$ and
  $m{\,:\;}\geq v$, respectively. Formally $\mathbb{I}(\mathbb{D}) \coloneqq
  (G,N,J)$, where  $$N\coloneqq\{m:\;\leq v\mid m\in M,v\in m^{\mathbb{D}}\}
  \cup\{m:\;\geq v\mid m\in M,v\in m^{\mathbb{D}}\}$$
  and
  $$ (g, m:\;\leq v)\in J:\iff m(g)\leq v,\qquad(g, m:\;\geq v)\in J:\iff m(g)\geq v.$$
 For simplicity, attributes which apply to all objects are usually omitted.
\end{definition}

The context derived from our introductory example (\cref{fig:tennis-dt})
via plain interordinal scaling is shown in \cref{fig:scaling}.
The linear orders of the value sets were already listed at the end of
\cref{sec:treebased}.
The conceptual scales are shown on the left. 
For example, in this derived context (on the right) object No.~5 has a
cross for the derived attribute \emph{``Temperature $\le$
  \text{mild}''} because in the original data the temperature value
for that object is \emph{``cool''}, and \emph{cool $\le$ mild}.

%   whose value domains for every $m\in M$ are linearly ordered, we
%   define the \emph{interordinal scale} of $\context$ by introducing
%   the set $N\coloneqq \{(m,\Box,w)\in M\times\{\leq,\geq\}\times W\mid
%   w\in W(m)\}$ as follows:
% \begin{equation*}
%   \label{eq:3}
%   \mathbb{I}(\context)\coloneqq (G,N,J) \text{, where }
%   (g,(m,\Box,w))\in J \text{ iff } m(g)\mathop{\Box} w
% \end{equation*}
  
% We want to provide a bit of explanation for this definition. To do so,
% we may refer the reader to 
% The objects within a decisions tree can be perceived using intervals
% on the respective attribute domains, cf. \cref{fig:tennis-dt}. These
% intervals can be read from the predicates annotated to the nodes of
% the decision tree. Note that these predicates correspond to the
% attributes of the interordinal scale. Any given decision tree that is
% created over the same domains can be represented within the
% interordinal scale $\mathbb{I}(\context)$. Conversly, the
% interordinal scale contains every interval that can be used by a
% decision tree. 

% Referring to~\cref{problem:treescale} we may note that interordinal
% scale for any  $\check G\neq G$ can be computed analogously. 
% The interordinal scale is fine, compared to the other scales. Hence,
% its concept lattice is typically comparatively large and therefore
% computationally intractable for real-world applications. However, it
% may serve as upper bound in terms of the size of the scaled
% context. 

All just introduced scaling have in common, that their respective
concept lattices are atomistic. Moreover, in all lattices are the
atoms are given by the set of all object concepts. This observation
depends on two assumptions, a) there are no missing values for any of
the objects viewed by a scaling (cf. complete many-valued context) and
b) every leaf node is supported by an object $g\in G$.

\subsection{Predicate Views}
\label{sec:predscale}
The so far presented approaches for scaling tree-based classifiers do
not use the predicates as attribute sets. Yet, these predicates are
essential for human interpretation. Hence, we introduce in the
following two novel conceptual views, i.e., conceptual scalings, that
can also be extracted from the decision tree. However, in contrast
to~\cref{def:leafscale,def:treescale} will the attribute set be
comprised of the predicates $\mathcal{P}(\mathcal{T})$ instead of the
tree nodes.

In order to do this, we introduce an intermediate structure, in detail
a formal context, which takes the place of the annotation function
$\phi$ of $\mathcal{T}$. This context is defined by
$\mathbb{P}(\mathcal{T})\coloneqq
(\mathcal{T},\mathcal{P}(\mathcal{T}), J)$ where $(n,P)\in J$ iff
there exists a $n\leq_{\mathcal{T}} h$ with $\phi(h)=P$. That is, a
node $n$ of the tree $\mathcal{T}$ is in incidence with a predicate
$P\in\mathcal{P}(\mathcal{T})$ if and only if $P$ is annotated to $n$,
or a predecessor of $n$. In the following this context is called the
\emph{predicate view of $\mathcal{T}$}.We want to hint why this
structure enables a (formal) interpretation of $\mathcal{T}$ by means
of the predicates. For any node $n$ that is on the decision path of an
object $g$, we have that $g\models \{n\}^{J}$, i.e., $g$ is a model
for all predicates that are incident with $n$. Moreover, for the leaf
node $l$ on the decision path of $g$, the set $\{l\}^{J}$ is exactly
the set of predicates that were used by $\mathcal{T}$ to classify the
object $g$. Hence, we can interpret the classification for any object
in terms of $\mathcal{P}(\mathcal{T})$.

\subsubsection*{Tree Predicate View}
\begin{definition}[Conceptual Tree Predicate View]
  \label{def:treepredicatescale}
  For a many-valued context $\mathbb{D}\coloneqq (G,M,W,I)$ and a
  decision tree $\mathcal{T}$ (trained on $\mathbb{D}$), we define the
  context derived from \emph{contextual tree predicate view} on
  $\mathcal{T}$ by
  \begin{equation*}
    \label{eq:4}
    \mathbb{T}_{\mathcal{P}}(G, \mathcal{T})\coloneqq
    (G,\mathcal{P}(\mathcal{T}),I_{\mathbb{T}(G, \mathcal{T})}\circ I_{\mathbb{P}(\mathcal{T})}).
  \end{equation*}
  Analogously to~\cref{def:leafscale} we say \emph{conceptual tree
    predicate view} on $\mathcal{T}$ to $\BV(
  \mathbb{T}_{\mathcal{P}}(G, \mathcal{T}))$. Likewise, this view can
  be applied to unknown data $\check G$.
\end{definition}

This definition of a view differs slightly from the ones given in the
last section. First of all, the attribute set of the tree predicate
view is comprised of the predicates of $\mathcal{T}$. Moreover, the
incidence relation of $\mathbb{T}_{\mathcal{P}}(\check G,\mathcal{T})$
is implicitly given by the relation product $\circ$ of the incidences
from the tree view and the predicate view. Hence the name tree
predicate view.  Our reasoning here is that we want to link objects
to predicates via tree nodes. For example, if $g\in\check G$ is
incident with node $n\in\mathcal{T}$, and again $n$ is incident with
some predicate $P\in\mathbb{P}(\mathcal{T})$, then $g$ is incident
with $P$ in $I_{\mathbb{T}(\check G, \mathcal{T})}\circ
I_{\mathbb{P}(\mathcal{T})}$. In contrast to the tree view, objects
that have a disjoint decision path may still have predicates in
common, and therefore common incidences in
$I_{\mathbb{T}_{\mathcal{P}}(\check G, \mathcal{T})}$.

The concept lattice of tree predicate view
$\mathbb{T}_{\mathcal{P}}(\check G, \mathcal{T})$ is not necessarily
tree shaped, since additional concepts may emerge from the meet of
predicates that were annotated multiple times. In particular for the
case where $\check G=G$, we find that
\begin{equation}
  \label{eq:treepredicatescale}
  \Ext(\mathbb{T}(G, \mathcal{T}))\subseteq
  \Ext(\mathbb{T}_{\mathcal{P}}(G, \mathcal{T})),
\end{equation}
since for all $n\in \mathcal{T}$ we have that
\[n^{I_{\mathbb{T}(G,\mathcal{T})}}=\{\phi(m)\mid m\geq
  n\}^{I_{\mathbb{T}_{\mathcal{P}}(G,\mathcal{T})}}.\]

To see why this is true, we refer the reader to~\cref{lamma-scales}.
Another useful property we prove in~\cref{lamma-scales} for
$\mathbb{T}_{\mathcal{P}}(\mathcal{T})$ is that its concept lattice is
atomistic and its object extents are equal to those of
$\mathbb{T}(\mathcal{T})$. A natural consequence of this fact is, that
conceptual view's concept lattice can classify every object $g\in G$
in the same way as the tree classifier would. That is, for any
$g\in G$ the closure of $g$ in the tree predicate view is equal to the
closure in the leaf view, i.e., the set of objects that are classified
by the same leaf. This fact is even true for any $g\in \check G$ that
has no missing values. What is more important, the tree predicate view
exhibits concepts that are not related to a node of $\mathcal{T}$,
however, they explain how different nodes of $\mathcal{T}$ are related
in terms of common predicates. This is the reason for the super set
relation in \cref{eq:treepredicatescale}.

In the case that  all predicates are annotated exactly once in the tree then the
annotate function is injective, i.e., for any two $n,m\in \mathcal{T}$
it is true that $\phi(n)=\phi(m)\implies n=m$. Here, we find that $\Ext(\mathbb{T}(\check G, \mathcal{T}))=
\Ext(\mathbb{T}_{\mathcal{P}}(\check G,\mathcal{T}))$.

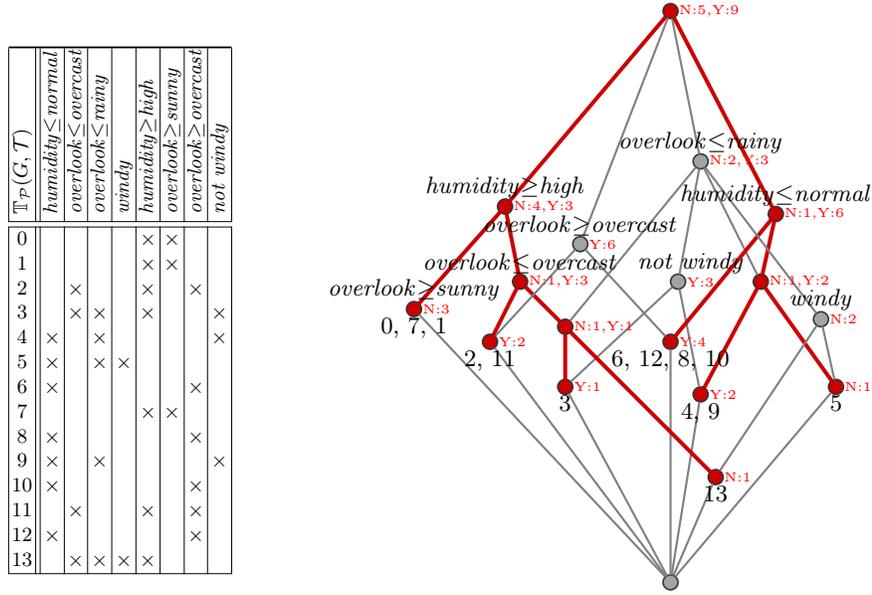
\begin{figure}[t]
  \newcommand{\scale}{.85}
  \newcommand{\ro}{90}
\begin{minipage}{.33\linewidth}
  \scalebox{\scale}{\begin{tabular}{|c||c|c|c|c|c|c|c|c|}
    \bottomrule
\rotatebox{\ro}{$\mathbb{T}_{\mathcal{P}}(G,\mathcal{T})$}     &\rotatebox{\ro}{\emph{humidity}$\leq$\emph{normal}}
                    &\rotatebox{\ro}{\emph{overlook}$\leq$\emph{overcast}}
                    &\rotatebox{\ro}{\emph{overlook}$\leq$\emph{rainy}}
                    &\rotatebox{\ro}{\emph{windy}}
                    &\rotatebox{\ro}{\emph{humidity}$\geq$\emph{high}}
                    &\rotatebox{\ro}{\emph{overlook}$\geq$\emph{sunny}}
                    &\rotatebox{\ro}{\emph{overlook}$\geq$\emph{overcast}}
                    &\rotatebox{\ro}{\emph{not windy}}\\
    \hline \hline
    0&&&&&$\times$&$\times$&&\\
    1&&&&&$\times$&$\times$&&\\
    2&&$\times$&&&$\times$&&$\times$&\\
    3&&$\times$&$\times$&&$\times$&&&$\times$\\
    4&$\times$&&$\times$&&&&&$\times$\\
    5&$\times$&&$\times$&$\times$&&&&\\
    6&$\times$&&&&&&$\times$&\\
    7&&&&&$\times$&$\times$&&\\
    8&$\times$&&&&&&$\times$&\\
    9&$\times$&&$\times$&&&&&$\times$\\
    10&$\times$&&&&&&$\times$&\\
    11&&$\times$&&&$\times$&&$\times$&\\
    12&$\times$&&&&&&$\times$&\\
    13&&$\times$&$\times$&$\times$&$\times$&&&\\
    \hline
  \end{tabular}}
\end{minipage}
\begin{minipage}{.3\linewidth}
  \colorlet{mivertexcolor}{black!80}
\colorlet{jivertexcolor}{black!20!red}
\colorlet{vertexcolor}{black!80}
\colorlet{bordercolor}{black!80}
\colorlet{linecolor}{gray}
\colorlet{dtreecolor}{jivertexcolor}
\tikzset{vertexbase/.style={semithick, shape=circle, inner sep=2pt, outer sep=0pt, draw=bordercolor},%
  vertex/.style={vertexbase, fill=vertexcolor!45},%
  mivertex/.style={vertexbase, fill=mivertexcolor!45},%
  jivertex/.style={vertexbase, fill=jivertexcolor},%
  divertex/.style={vertexbase, top color=mivertexcolor!45, bottom color=jivertexcolor!45},%
  dtree/.style={-, line width=1.5pt, color=dtreecolor},
  conn/.style={-, thick, color=linecolor}%
}
\tikzstyle{v} = [font=\tiny, color=red]
\tikzstyle{l} = [align=left]
\begin{tikzpicture}[scale=0.2]
  \begin{scope} %for scaling and the like
    \begin{scope} %draw vertices
      \foreach \nodename/\nodetype/\xpos/\ypos in {%
        0/vertex/0.0/0.0,
        1/jivertex/3.0/7.0,
        2/jivertex/-7.0/13.0,
        3/jivertex/11.0/13.0,
        4/jivertex/-12.0/16.0,
        5/jivertex/0.0/16.0,
        6/jivertex/2.0/12.5,
        7/jivertex/-7.0/17.0,
        8/vertex/10.0/17.5,
        10/jivertex/-17.05216554379211/18.175264677574614,
        11/jivertex/6.0/20.0,
        12/jivertex/-10.0/20.0,
        13/vertex/0.5/20.0,
        15/vertex/-6.0/22.5,
        16/jivertex/-11.0/25.0,
        17/jivertex/7.0/24.5,
        18/vertex/2.0/28.0,
        19/jivertex/0.0/38.0
      } \node[\nodetype] (\nodename) at (\xpos, \ypos) {};
    \end{scope}
    \begin{scope} %draw connections
      \path (0) edge[conn] (2);
      \path (13) edge[conn] (18);
      \path (18) edge[conn] (19);
      \path (7) edge[conn] (18);
      \path (7) edge[dtree] (12);
      \path (0) edge[conn] (6);
      \path (4) edge[conn] (15);
      \path (3) edge[dtree] (11);
      \path (6) edge[conn] (13);
      \path (11) edge[conn] (18);
      \path (11) edge[dtree] (17);
      \path (17) edge[dtree] (19);
      \path (5) edge[dtree] (17);
      \path (15) edge[conn] (19);
      \path (6) edge[dtree] (11);
      \path (0) edge[conn] (3);
      \path (0) edge[conn] (5);
      \path (4) edge[dtree] (12);
      \path (12) edge[dtree] (16);
      \path (0) edge[conn] (10);
      \path (0) edge[conn] (4);
      \path (0) edge[conn] (1);
      \path (10) edge[dtree] (16);
      \path (16) edge[dtree] (19);
      \path (2) edge[dtree] (7);
      \path (5) edge[conn] (15);
      \path (3) edge[conn] (8);
      \path (8) edge[conn] (18);
      \path (2) edge[conn] (13);
      \path (1) edge[conn] (8);
      \path (1) edge[dtree] (7);
    \end{scope}
    \begin{scope} %add labels
      \foreach \nodename/\labelpos/\labelopts/\labelcontent in {%
        1/below//{13},
        2/below//{3},
        3/below//{5},
        4/below//{2, 11},
        5/below//{6, 12, 8, 10},
        6/below//{4, 9},
        8/above//{\emph{windy}},
        10/below//{0, 7, 1},
        10/above//{\emph{overlook}$\geq$\emph{sunny}},
        12/above//{\emph{overlook}$\leq$\emph{overcast}},
        13/above//{\phantom{xx}\emph{not windy}},
        15/above//{\emph{overlook}$\geq$\emph{overcast}},
        16/above//{\emph{humidity}$\geq$\emph{high}},
        17/above//{\emph{humidity}$\leq$\emph{normal}},
        18/above//{\emph{overlook}$\leq$\emph{rainy}},
        1/right/v/{N:1},
        2/right/v/{Y:1},
        3/right/v/{N:1},
        4/right/v/{Y:2},
        5/right/v/{Y:4},
        6/right/v/{Y:2},
        7/right/v/{N:1,Y:1}, 
        8/right/v/{N:2},
        10/right/v/{N:3},
        11/right/v/{N:1,Y:2},
        12/right/v/{N:1,Y:3},
        13/right/v/{Y:3},
        15/right/v/{Y:6},
        16/right/v/{N:4,Y:3},        
        17/right/v/{N:1,Y:6},
        18/right/v/{N:2,Y:3},
        19/right/v/{N:5,Y:9}
      } \coordinate[label={[\labelopts]\labelpos:{\labelcontent}}](c) at (\nodename);
    \end{scope}
  \end{scope}
\end{tikzpicture}
\end{minipage}
\caption{The contextual and conceptual tree predicate view on
  $\mathcal{T}$ (right in \cref{fig:tennis-dt}) for the running example
  (left in \cref{fig:tennis-dt}). The original decision tree
  $\mathcal{T}$ is highlighted in red in the concept lattice diagram.}
  \label{fig:tennis-pred}
\end{figure}

\subsubsection*{Interordinal Predicate View}
The just introduced tree predicate view is capable of reflecting the
predicates that are important for the specific classification of an
object $g$. However, the incidences of $g$ are limited to those
predicates annotated to the decision path of $g$. In the view to be
introduced in a moment we want to lift this restriction by extending
the incidence relation to all to all predicates
$P\in\mathcal{P}(\mathcal{T})$ for which $g$ is a model ($g\models
P$).

\begin{definition}[Conceptual Interordinal Predicate View]
  \label{defi:interordinalpredicatescale}
  For a mv-context $\mathbb{D}\coloneqq (G,M,W,I)$ and a decision tree
  $\mathcal{T}$ (trained on $\mathbb{D}$), we define the
  \emph{conceptual interordinal predicate view} on $\mathcal{T}$ by
  \begin{equation*}
    \label{eq:5}
    \mathbb{I}_{\mathcal{P}}(G, \mathcal{T})\coloneqq(G,\mathcal{P}(\mathcal{T}),J), \text{ where } (g,P)\in J\
    \text{iff}\ g \models P.
  \end{equation*}
  Analogously to all previous definitions we say \emph{conceptual
    interordinal predicate view} on $\mathcal{T}$ to $\BV(
  \mathbb{I}_{\mathcal{P}}(G, \mathcal{T}))$. Likewise, this view can
  be applied to previously unknown data $\check G$.
\end{definition}

There are two principle ways to derive the interordinal predicate
view in our setting. The natural way is to apply \emph{plain
  scaling}, which results in a sub-context of the interordinal
scaling of $\mathbb{D}$, i.e., $\mathbb{I}_{\mathcal{P}}(G, \mathcal{T})\leq
\mathbb{I}(\mathbb{D})$. This can be easily seen, since
$\mathcal{P}(\mathcal{T})\subseteq N$, where $N$ is defined as
in~\cref{def:interordinal}. Note, this definition requires the value
domains to be linearly ordered. The second method is to employ logical
scaling where each logical expression consists of conditions involving
exactly one many valued attribute. We depicted the interordinal
predicate view for our running example in~\cref{fig:tennis-inter}.

Since the interordinal predicate view is defined on the same set of
objects and attributes as the tree predicate views, both views are
related. More precisely, we find that the incidence relation of the
tree predicate view is a subset of the interordinal predicate view,
i.e., $I_{\mathbb{T}_{\mathcal{P}}(\check G,\mathcal{T})}\subseteq
I_{\mathbb{I}_{\mathcal{P}}(\check G,\mathcal{T})}$. 

With the following proposition we want to show how the classical and
our novel views are related to each other with respect to their
training data set, i.e., $\check G = G$. From this we can infer how,
or to which extent, the views can be employed for the explanation of
tree based classifiers. 

\begin{figure}[t]
  \newcommand{\scale}{.85}
  \newcommand{\ro}{90}
  \centering
  \scalebox{\scale}{\begin{tabular}{|c||c|c|c|c|c|c|c|c|}
    \bottomrule
\rotatebox{\ro}{$\mathbb{I}_{\mathcal{P}}(\mathcal{T})$} &\rotatebox{\ro}{\emph{humidity}$\leq$\emph{normal}}
                    &\rotatebox{\ro}{\emph{overlook}$\leq$\emph{overcast}}
                    &\rotatebox{\ro}{\emph{overlook}$\leq$\emph{rainy}}
                    &\rotatebox{\ro}{\emph{windy}}
                    &\rotatebox{\ro}{\emph{humidity}$\geq$\emph{high}}
                    &\rotatebox{\ro}{\emph{overlook}$\geq$\emph{sunny}}
                    &\rotatebox{\ro}{\emph{overlook}$\geq$\emph{overcast}}
                    &\rotatebox{\ro}{\emph{not windy}}\\
    \hline \hline
    0&&&&&$\times$&$\times$&$\times$&$\times$\\
    1&&&&$\times$&$\times$&$\times$&$\times$&\\
    2&&$\times$&&&$\times$&&$\times$&$\times$\\
    3&&$\times$&$\times$&&$\times$&&&$\times$\\
    4&$\times$&$\times$&$\times$&&&&&$\times$\\
    5&$\times$&$\times$&$\times$&$\times$&&&&\\
    6&$\times$&$\times$&&$\times$&&&$\times$&\\
    7&&&&&$\times$&$\times$&$\times$&$\times$\\
    8&$\times$&&&&&$\times$&$\times$&$\times$\\
    9&$\times$&$\times$&$\times$&&&&&$\times$\\
    10&$\times$&&&$\times$&&$\times$&$\times$&\\
    11&&$\times$&&$\times$&$\times$&&$\times$&\\
    12&$\times$&$\times$&&&&&$\times$&$\times$\\
    13&&$\times$&$\times$&$\times$&$\times$&&&\\
    \hline
  \end{tabular}}
\caption{The contextual interordinal predicate view on $\mathcal{T}$
  (right in \cref{fig:tennis-dt}) for the running example (left in
  \cref{fig:tennis-dt}). Its concept lattice contains 55 concepts.}
  \label{fig:tennis-inter}
\end{figure}

\begin{proposition}\label{lamma-scales}
  Let $\mathbb{D}$ be a complete mv-context and let $\mathcal{T}$ be a
  decision tree, which was trained on $\mathbb{D}$, such that every
  leaf node of $\mathcal{T}$ is supported. Then the following
  statements hold:

  \begin{enumerate}[i)]
  \item in every view $\mathbb{N}(G,\mathcal{T}),
    \mathbb{T}(G,\mathcal{T}), \mathbb{T}_{\mathcal{P}}(G,
    \mathcal{T})$, and $\mathbb{I}_{\mathcal{P}}(G, \mathcal{T})$ we find that the object concepts are the
    atoms of their respective concept lattice,
  \item $\Ext(\mathbb{N}(G,\mathcal{T}))\subseteq
    \Ext(\mathbb{T}(G,\mathcal{T}))\subseteq
    \Ext(\mathbb{T}_{\mathcal{P}}(G,\mathcal{T}))$ and\newline 
    $\Ext(\mathbb{T}(G,\mathcal{T}))\subseteq
    \Ext(\mathbb{I}_{\mathcal{P}}(G,\mathcal{T}))$
  \item $I_{\mathbb{T}_{\mathcal{P}}(G,\mathcal{T})}\subseteq
    I_{\mathbb{I}_{\mathcal{P}}(G,\mathcal{T})}$,
  \item the object extents of $\mathbb{N}(G,\mathcal{T}),
    \mathbb{T}(G,\mathcal{T})$ and $
    \mathbb{T}_{\mathcal{P}}(G,\mathcal{T})$ are equal,
    
  \item and if the value domains of the attributes $M$ are linearly
    ordered we also find that the object concepts of
    $\mathbb{I}(\mathbb{D})$  are the atoms of its respective concept
    lattice and that   $\Ext(\mathbb{I}_{\mathcal{P}}(G,\mathcal{T}))\subseteq
  \Ext(\mathbb{I}(\mathbb{D}))$ holds. 
  \end{enumerate}
\end{proposition}
\begin{proof}
  \begin{inparaenum}[i)]
  \item For nominal scales it is a known fact the set of atoms is
    comprised of object concepts. Since the introduced leaf view is
    of nominal scale we can infer the statement to be true.

    In case of the tree view $\mathbb{T}(G,\mathcal{T})$ we know that
    $\{g\}^{J}$ is equal to the set of nodes from $\mathcal{T}$ that
    are on the (complete) decision path of $g$. If we assume that
    there is a formal concept below $(\{g\}^{JJ},\{g\}^{J})$, then
    there must exist a $h\in G$ with $\{g\}^{J}\subseteq \{h\}^{J}$
    and $\{g\}^{J}\neq \{h\}^{J}$. However, this would imply that the
    decision path of $g$ can be extended, which is a contradiction. We
    may note that in case of an incomplete many-valued context
    $\mathbb{D}$, this argument does not hold.

    For the tree predicate view we can apply the same argument.

    For the interordinal predicate view, assume there are two objects
    $g,h\in G$ with $\{g\}'\subseteq \{h\}'$ and $\{g\}'\neq\{h\}'$,
    then $\{h\}'$ is not an atom in $\BV(\mathbb{I}_{\mathcal{P}}(G,
    \mathcal{T}))$. Note, the derivation $(\cdot)'$ is taken with
    respect to $\mathbb{I}_{\mathcal{P}}(G, \mathcal{T})$.  Hence,
    there is a predicate $P\in \{h\}'$ with $h\models P$ and $g\not\models
    P$. Due to the completeness that arises from the complete
    mv-context, we can infer that $g\models\neg P$. Therefore, $\neg
    P\in\{g\}'$, which contradicts the assumption.

  \item The set of leaf nodes $\mathcal{L}(\mathcal{T})$ is a subset
    of all nodes in $\mathcal{T}$. From this fact we can infer that
    $\mathbb{N}(G,\mathcal{T})$ is an attribute induced subcontext of
    $\mathbb{T}(G, \mathcal{T})$ and thus
    $\Ext(\mathbb{N}(G,\mathcal{T}))\subseteq
    \Ext(\mathbb{T}(G,\mathcal{T}))$, see context
    apposition~\cite{fca-book}.

    Next we want to show $\Ext(\mathbb{T}(G,\mathcal{T}))\subseteq
    \Ext(\mathbb{T}_{\mathcal{P}}(G,\mathcal{T}))$. For a nonempty
    extent $A\subseteq G$ of $\mathbb{T}(G,\mathcal{T})$ we know its
    derivation in said view is a path in the decision tree from the
    root up to some node $n$.  For a decision tree, this path is
    uniquely identified by the annotated predicates, since the
    predicates $\mathcal{Q}\subseteq\mathcal{P}(\mathcal{T})$ of the
    split in the root node cannot be annotated twice. Otherwise, this
    would lead to an unsupported leaf, which contradicts the
    requirements of the proposition. The derivation of $\mathcal{Q}$
    in $\mathbb{T}_{\mathcal{P}}(G, \mathcal{T})$ is equal to $A$,
    since all objects that pass through node $n$ are exactly those
    that are a model of $Q$. It remains to be shown that the empty set
    is an extent on both views, which follow from i) and the fact
    that the decision tree $\mathcal{T}$ has at least one
    split.\footnote{We made this requirement for all decision trees
      that are considered in this work in~\cref{sec:formalities}.}
    Thus $\Ext(\mathbb{T}(G,\mathcal{T}))\subseteq
    \Ext(\mathbb{T}_{\mathcal{P}}(G,\mathcal{T}))$.

    The same argument can be applied to
    $\mathbb{I}_{\mathcal{P}}(G,\mathcal{T})$.

    For any path from the root node to some other node
    $n\in\mathcal{T}$ let $\mathcal{Q}\subseteq
    \mathcal{P}(\mathcal{T})$ be the set of annotated predicates. Then
    the derivation of $\mathcal{Q}$ in
    $\mathbb{T}_{\mathcal{P}}(G,\mathcal{T})$ is equal to the
    derivation of $\mathcal{Q}$ in
    $\mathbb{I}_{\mathcal{P}}(G,\mathcal{T})$. This is true since the
    set of objects that passes through the node $n$, i.e.,
    $\mathcal{Q}^{I_{\mathbb{T}_{\mathcal{P}}(G,\mathcal{T})}}$, is
    given by the set of objects that models $\mathcal{Q}$, i.e.,
    $\mathcal{Q}^{I_{\mathbb{I}_{\mathcal{P}}(G,\mathcal{T})}}$. Thus
    $\Ext(\mathbb{T}(G,\mathcal{T}))\subseteq
    \Ext(\mathbb{I}_{\mathcal{P}}(G,\mathcal{T}))$.

  \item Follows directly from their definitions.

  \item Any object extent $A\subseteq G$ of
    $\mathbb{N}(G,\mathcal{T})$ has the property that there is a
    unique leaf node $l\in\mathcal{T}$ such that $A$ is the set of
    objects that is classified by $l$.

    From i) we can infer that the object extents of
    $\mathbb{T}_{\mathcal{P}}(G,\mathcal{T})$ are the atoms in the
    concept lattice
    $\BV(\mathbb{T}_{\mathcal{P}}(G,\mathcal{T}))$. For any object
    extent $A$ of $\mathbb{N}(G,\mathcal{T})$, let $g\in G$ be a
    generator of $A$, i.e.,
    $\{g\}^{I_{\mathbb{N}(G,\mathcal{T})}I_{\mathbb{N}(G,\mathcal{T})}}=A$.

    We can find the associated set of predicates in the tree predicate
    view by computing the derivation of $g$ in
    $\mathbb{T}_{\mathcal{P}}(G,\mathcal{T})$. In detail, we can
    compute $\{g\}^{I_{\mathbb{T}_{\mathcal{P}}(G,\mathcal{T})}}$ by
    projecting
    \[\{(g,n)\mid n\ \text{is on the dec.  path of}\ g
      \}\circ\{(n,P)\mid \exists m\in\mathcal{T}:
      n\leq_{\mathcal{T}}m\ \text{and}\ \phi(m)=P\}\] on the second
    element. This set is equal to the set of predicates annotated to
    the decision path of $g$. The second application of the derivation
    operation in $\mathbb{T}_{\mathcal{P}}(G,\mathcal{T})$ yields the
    set of objects that are model of
    $\{g\}^{I_{\mathbb{T}_{\mathcal{P}}(G,\mathcal{T})}}$. Hence, all
    elements of
    $\{g\}^{I_{\mathbb{T}_{\mathcal{P}}(G,\mathcal{T})}I_{\mathbb{T}_{\mathcal{P}}(G,\mathcal{T})}}$
      are classified by the same leaf as $g$. Thus, the object extents
      of $\mathbb{N}(G,\mathcal{T})$ and
      $\mathbb{T}_{\mathcal{P}}(G,\mathcal{T})$ are equal. The rest of
      the statements follows directly from ii).

    \item  It is a known fact for interordinally scaled complete many-valued
    contexts that the object concepts are the atoms of the respective
    concept lattice.
    
    Furthermore, for linear ordered value domains we find that
    $\mathcal{P}(\mathcal{T})\subseteq N$ where
    $N\coloneqq\{m:\;\leq v\mid m\in M,v\in m^{\mathbb{D}}\}
    \cup\{m:\;\geq v\mid m\in M,v\in m^{\mathbb{D}}\}$, we can infer
    that $\mathbb{I}_{\mathcal{P}}(G, \mathcal{T})$ is an attribute
    induced sub-context of $\mathbb{I}(\mathbb{D})$. Hence, the
    statement holds.

    \end{inparaenum}\qed
\end{proof}

From this proposition we can draw essential consequences for the
conceptual interpretation of (or view on) tree classifiers. From ii)
we can infer that the extent structure is entailed in both the extent
structure of the tree predicate view and the interordinal predicate
view. Hence, the whole decision tree structure is captured by both
views. To demonstrate this within the scope of our running example,
we depicted the decision tree within the tree predicate view of the
training data $G$ in~\cref{fig:tennis-pred} (right). In addition to
the decision tree structure of $\mathcal{T}$, we can observe in
$\mathbb{T}_{\mathcal{P}}(G,\mathcal{T})$ multiple predicate
combinations that span across different tree branches. This
theoretical finding leads to several approaches for the interpretation
of decision trees. We want to introduce and discuss these using our example. 
\begin{description}
\item[Alternative Leaf Descriptions:] Our method can generate
  alternative descriptions for leaf nodes in the predicate language of
  $\mathcal{T}$. For example, the leaf $l$ that classifies object 13
  in~\cref{fig:tennis-pred} contains the predicates
  \textit{humidity$\geq$high, overlook$\leq$overcast,
    overlook$\leq$rainy, windy}. This leaf has an upper neighbor within
  the concept lattice of the tree predicate view having the attribute
  \textit{windy}. There is no node within $\mathcal{T}$ representing
  this concept. However, we can use this concept to construct an
  alternative combination of predicates that generates the concept of
  $l$. The leaf $\l$ can be represented by the meet
  $\bigwedge\{\textit{windy, humidity$\geq$high}\}$. This is in fact
  a \emph{minimal generator} for the intent of the concept associated
  to $l$.

  In order to interpret a given decision tree $\mathcal{T}$, one can
  generate all minimal generators for all leaf concepts, and use
  these as shorter descriptions to comprehend the classification
  structure of the tree. The thereby obtained shorter explanations are
  potentially more comprehensible. This is in particular useful, when
  decision trees are large, for example, when trained on large data
  sets having a many attributes. 
\item[Explaining Leaf Sets:] For any set of leafs
  $L\subseteq\mathcal{T}$, we can compute in the tree predicate view
  their (conceptual) join $(A,B)\coloneqq\bigvee L$. The formal
  concept $(A,B)\in\BV(\mathbb{T}_{\mathcal{P}}(G,\mathcal{T}))$ is
  not necessarily associated to a node of $\mathcal{T}$, for example,
  take the join of the leafs having objects 13 and 5
  in~\cref{fig:tennis-pred}. From this we learn that both leafs share
  the predicate \textit{windy}. When we follow the lattice towards
  the top concept, we find that the leafs 13 and 5 also share the
  predicate \textit{overlook$\leq$rainy}. In contrast, within
  $\mathcal{T}$, the decision paths of both leafs have only the root
  node in common. Hence, our method is capable of expressing
  commonalities of the set of leafs $L$, that are inexpressible within
  the structure of $\mathcal{T}$.   
\item[Control for Missing Data:]
  The alternate descriptions from the previous two items allow for
  coping with missing attributes. For example, when classifying an
  object $g\in \check G$ that has no value for the attribute
  \textit{overlook}, the predicate tree view can map $g$ to the leaf
  having object 13 using the attributes \textit{windy} and
  \textit{humidity}, as discussed in the previous items. 
\item[Global Influence of a Predicate:] A common method for
  interpreting and explaining decision trees is to identify attributes
  that are used first or second in the tree. Yet, as the tree
  predicate view reveals, there are other structurally important
  predicates, e.g., \textit{overlook$\leq$rainy} and
  \textit{overlook$\geq$overcast}, as there is no upper neighbor for the
  associated concepts besides the root node. We say a formal concept
  $(A,B)$  \emph{is dominated by} another concept $(C,D)$ iff $(A,B)\leq
  (C,D)$. Based on this notion, we can say that a predicate $P$ is
  dominated by another predicate $Q$ iff the attribute concept of $P$
  is a lower neighbor of the attribute concept of $Q$, i.e.,
  ${P}^{J}\subseteq {Q}^{J}$.
\item[Leaf Coverage of Predicates:] Another measure of importance for
  an predicate $P$ within a decision tree $\mathcal{T}$ is the number
  of leafs that $P$ is involved with. It is not surprising that a
  predicate which is used first in the decision tree will be involved in many
  leaves. However, as we can infer from the lattice diagram
  in~\cref{fig:tennis-pred}, the predicate
  \textit{overlook$\leq$rainy} is involved in four leaves while its
  predecessor in the tree is only involved in three leaves. Hence, the
  conceptual view on $\mathcal{T}$ allows for structurally identifying
  important predicates. Moreover, one may easily select a subset of
  predicates that covers all leafs of a tree $\mathcal{T}$. 
\end{description}

The methodology just presented for the analysis and interpretation of
decision trees can be applied to the interordinal tree view in an
analogous way, as~\cref{lamma-scales} ii) points out. In general,
given the set inclusion on the extent sets, one can consider the
views $\mathbb{N}(G,\mathcal{T}),\mathbb{T}(G,\mathcal{T})$ as
coarse, the views
$\mathbb{T}_{\mathcal{P}}(G,\mathcal{T}),\mathbb{I}_{\mathcal{P}}(G,\mathcal{T})$
as intermediate, and the view $\mathbb{I}(\mathbb{D})$ as a fine
scaling of $\mathcal{T}$.

\subsection{Tree Ensemble Views}
\label{sec:ensscale}
In the last section, we introduced all tree views in a common
language. This enables us to compare their different views on a
decision tree, as we have seen. This comparability provides the
cornerstone for a comprehensive approach to the interpretation of tree
ensembles. In particular, we present a principle approach for the
interpretation of families of trees, as they are used in common
supervised machine learning procedures, such as Random Forests or
AdaBoost. In the following we will use the notation $\Scon(\check G,
\mathcal{T})$ as a general name for any of the views introduced in
the last sections.

\begin{definition}[Forest View]\label{def:union-view}
  For a mv-context $\mathbb{D}\coloneqq (G,M,W,I)$, a family of decision
  trees $\mathfrak{T}=(\text{T}_i)_{i\in F}$ that were trained on
  $\mathbb{D}$, and a conceptual view for each tree $\Scon(\check G,
  \text{T}_i)$, we define their \emph{forest view} to be
\[\Scon(G, \mathfrak{T})\coloneqq\bigcup_{i\in F}\Scon( G, \text{T}_i), \text{where }
  \Scon_1\cup\Scon_2\coloneqq (G_1\cup G_2,M_1\cup M_2,I_1\cup I_2).\]
Likewise, this view can be applied to previously unknown data
$\check G$.
\end{definition}

We want to elaborate on the reasoning behind this definition. First we
may note, that for our views the sets $G_{1}=G_{2}= G$, hence,
their union results in $G$. This is by design, since we want
our method to interpret a forest given a particular set of objects
$G$, analogously to the decision tree views.
%, cf.~\cref{problem?!}
In contrast, the
attribute sets and incidences can be different. Hence, the ability of
the forest view to provide an interpretation is directly connected to
the union of all attributes from the set of tree views.

\subsubsection*{Random Forests} are constructed using two essential
techniques, \emph{bagging} and an empirical variant of
\emph{boosting}. The term bagging, also known as bootstrap
aggregation, describes the procedure to draw samples uniformly from
the training data set, with replacement. The common approach is to
sample for any tree in an ensemble its own training data
set~\cite{breiman2001random}. In the language of our formal contexts,
each tree is constructed using a random induced subcontext
$\mathbb{H}\leq\mathbb{D}$. This modeling does not take into account the
possibility that the same object can be drawn twice, however, this can
obviously dealt with by creating copies of objects.
In the following we to discuss the influence of the different tree
views on the forest view.

\begin{description}
\item[Leaf View:] In this view the attribute set is comprised of the
  set of leafs $\mathcal{L}(\mathcal{T})$ for any tree $\mathcal{T}$
  in the ensemble $\mathfrak{T}$. For any two trees $\mathcal{T}_{1}$
  and $\mathcal{T}_{2}$, we consider their leaf sets
  $\mathcal{L}((\mathcal{T}_{1}))$ and $\mathcal{L}(\mathcal{T}_{2})$
  to be disjoint sets, i.e.,
  $\mathcal{L}(\mathcal{T}_{1})\cap\mathcal{L}(\mathcal{T}_{2})=\emptyset$.

  This implies that any forest view whose trees are the leaf scaled
  is equal to the context apposition of the set of leaf views
  $\mid_{i\in F} \Scon(\check G,\mathcal{T}_{i})=\bigcup_{i\in
    F}\Scon(\check G, \mathcal{T}_{i})$. From the context apposition
  we can deduce that the set of its extents, i.e.,  $\Ext(\mid_{i\in F}
  \Scon(\check G,\mathcal{T}_{i}))$,  is equal to the set of
  intersections of all subsets of extents from all tree views, i.e.,
  $\{\bigcap \mathcal{A}\mid \mathcal{A}\subseteq \bigcup_{i\in
    F}\Ext(\Scon(\check G, \mathcal{T}_{i}))\}$.
\item[Tree View:]
  Analogously to the modeling by the leaf view, we consider the nodes
  for any two trees $\mathcal{T}_{1},\mathcal{T}_{2}$ to be disjoint,
  i.e., $\mathcal{T}_{1}\cap \mathcal{T}_{2}=\emptyset$. Hence, the
  forest view is equal to the apposition of all tree views, having
  the same consequence on its extents as shown in the last item.  
\item[Tree Predicate View:] Forest views that are based on the tree
  predicate view are more complicated than the previous two. For
  example, given two trees $\mathcal{T}_{1},\mathcal{T}_{2}$, an
  object $g\in \check G$ and a predicate
  $P\in\mathcal{P}(\mathcal{T}_{1})\cap\mathcal{P}(\mathcal{T}_{2})$,
  the case may arise that $(g,P)\in
  I_{\mathbb{T}_{\mathcal{P}}(\check G,\mathcal{T}_{1})}$ but
  $(g,P)\not\in I_{\mathbb{T}_{\mathcal{P}}(\check
    G,\mathcal{T}_{2})}$. Therefore, the tree predicate view based
  forest view is not a simple apposition of its individual tree
  predicate views. Hence, it is possible that this forest view can
  come up with extents that were not simple intersection of already
  known extents. An advantage of employing tree predicate views
  for constructing the forest view is that the resulting
  representation is smaller. This is due to the fact that any two
  trees might share predicates, i.e.,
  $\mathcal{P}(\mathcal{T}_{1})\cap\mathcal{P}(\mathcal{T}_{2})\neq\emptyset$,
  and in particular potentially large.
\item[Interordinal Predicate View:] In contrast to the last view, we
  can state for the interordinal predicate view on Random Forests that
  given two trees $\mathcal{T}_{1},\mathcal{T}_{2}$, an object $g\in
  \check G$ and a predicate
  $P\in\mathcal{P}(\mathcal{T}_{1})\cap\mathcal{P}(\mathcal{T}_{2})$
  we find $(g,P)\in I_{\mathbb{I}_{\mathcal{P}}(\check
    G,\mathcal{T}_{1})}$ iff $(g,P)\in
  I_{\mathbb{I}_{\mathcal{P}}(\check G,\mathcal{T}_{2})}$. From this
  we can infer that the forest view is almost the apposition of the
  set of interordinal predicate views $\mathbb{I}_{\mathcal{P}}(\check
  G,\mathcal{T}_{i})$, with the exception that any predicate $P$ that
  occurs in more than one view is not duplicated by coloring. Since
  clarification of attributes, i.e., the removal of duplicate
  attributes in a formal context, does not affect the set of extents, we
  can apply the same reasoning as shown for the leaf view.
\end{description}

The views just introduced open up a variety of practical applications
for the explainability of tree ensembles, which we will study in more
detail in~\cref{sec:post-scaling}.

\section{Dealing with Large Conceptual Views}\label{sec:post-scaling}
We discussed in \cref{sec:predscale,sec:ensscale} the utility and
applicability of the different conceptual views. Yet, for many
examples of real-world sized data these views are potentially
incomprehensibly large.
Thus, in order to derive human-comprehensible selections and
aggregations of the conceptual views, we introduce the following
methods. These methods are based on common data reduction procedures
for formal contexts, however, adapted for conceptual views.

\subsubsection*{Object or Attribute Selection}
The first class of methods are selection methods to compute induced
subcontext of conceptual views. Selecting a subset of the object set
will result in a coarser closure system \cite{smeasure} on the set of
attributes. A selection of attributes of the contextual view has the same effect
on the objects~\cite{smeasure}. There are numerous ways on how to
select relevant attributes from formal
contexts~\cite{HanikaKS19,DurrschnabelKS21}. In our experiments
(\cref{sec:expstudy}), we employ the feature importance scores that
are provided by the Random Forest models. Furthermore, one may apply
\texttt{KMedoid} clustering to identify representative objects, which
we call \emph{center objects}. The advantage of \texttt{KMedoid}
compared to other popular methods, such as \texttt{kmeans}, is that
the cluster centers are existing objects of the data set. Thus, the
clustering can be interpreted as computing an induced sub-context with
a subset of the original object set.

\subsubsection*{ Structure based Object Selection}
A particular method for selecting objects can be based on the
structural position of an object within the concept lattice.  For a
given object $g\in G$, a natural approach would be to compute the
order filter ${\uparrow}\{g\}^{I_{\Scon}I_{\Scon}}\subseteq \BV(\Scon)$
for a contextual view $\Scon$. This results in a \emph{local
  conceptual view} that allows for deriving explanations for
individual objects.  A second approach additionally includes
neighboring concepts of ${\uparrow}\{g\}^{I_{\Scon}I_{\Scon}}$. The
resulting local conceptual view enables more comprehensive
explanations of the structural position of $g$ within $\mathcal{T}$,
and its dependence from different attribute values. In particular,
this allows for investigations that are comparable to \emph{partial dependence
  plots}~\cite{hastie2009elements}, i.e., it enables the study of
attribute value perturbations.

Neighboring concepts can be added using covering elements of the set
${\uparrow}\{g\}^{I_{\Scon}I_{\Scon}}$, i.e., $\{A\in\BV(\Scon) \mid
A\prec B \vee B \prec A \text{ for } B\in
{\uparrow}\{g\}^{I_{\Scon}I_{\Scon}}\}$. Those elements can be
enumerated recursively using the \texttt{next\_neighbor}
algorithm~\cite{next-neighbor}.

\subsubsection*{Concept Selection Methods}
To reduce the number of formal concepts, it is common to apply
different criteria for their importance. The FCA literature provides a
multitude of measures~\cite{measures}. In our experimental
work, we select concepts based on their support~\cite{titanic} (TITANIC), i.e.,
the number of objects that are contained in an extent divided by the
number of all objects.  This procedure results in a subset of the
concepts of a conceptual view. This set constitutes a
join-semilattice, i.e., the \emph{iceberg concept lattice}.

\subsubsection*{Composition Methods}
Another approach is to split a conceptual view into multiple parts
based on a given partition of the object or attribute set. The
original concepts of a conceptual view can be retrieved from the
individual parts by combining them using the meet and joins
operations. Reasonable partitions of the object set can be derived
using their class labels. Hence, from this one can compute a drawing
per class label. A meaningful choice for a partition of the attributes
is to draw on their semantics. For example, employing ontological
background knowledge. Furthermore, one may restrict a view to a
particular \emph{order direction} of the threshold values, i.e.,
$\leq$ and $\geq$. This procedure can be considered as an ordinal
factors with respect to the context apposition
operation~\cite{fca-book}.

\subsubsection*{Attribute Aggregation} % Number of values
A reason for why the number of concepts of a conceptual view gets
large is the number of different predicates derived during the
training.  Aggregating different predicates by clustering them may
lead to a significant reduction in the number of concepts. For this,
one should account for the different attribute value distributions on
which the predicates are based on. A clustering using \emph{grades} as
aggregated values, e.g., \emph{low}$\leq$\emph{med}$\leq$\emph{high},
can be especially comprehensible to human readers.

\section{Experimental Study}
\label{sec:expstudy}
The following experiments shall support our theoretical findings with
respect to two practical research questions. First, is the size of
conceptual views manageable with respect to human-comprehensibility
and to what extent depends its size on the choice of hyper
parameters of the tree training algorithm? Second, are explanations
derived from conceptual views meaningful for human-understanding?

To answer these questions we conduct two experiments using Random
Forests. As for a data set, we choose the well-known car data
set~\cite[ID:991]{OpenML2020}, which is comprised of 1728 objects on
seven (many-valued) attributes. This dataset presents a binary
classification problem, using the class labels \texttt{positive} and
\texttt{negative}.

\subsection{Sizes of Conceptual Views: a Parameter Study}
We investigate the first research question by means of a parameter
study. The two most important hyper parameters of the Random Forest
procedure are the \emph{number of trees} (\texttt{nt}) and their
\emph{maximal depth} (\texttt{md}). Other parameters, such as
attributes per tree, purity, split criterion, etc, also have a
significant influence, however, not on the size of the resulting
conceptual view.

For our study, we trained different Random Forest classifiers using
$2\leq\mathtt{nt}\leq 10$ and $2\leq\mathtt{md}\leq 20$. We completed
ten runs for each parameter combination, using ten different initial
random seeds. In \cref{fig:ps-acc} (left) we report the classification
performance using the average accuracy and in \cref{fig:ps-acc}
(right) the generalization error. The latter is comprised of
subtracting the accuracy on the test data set from the accuracy that
was achieved on the train data set, i.e.,
$\text{Error}_{\text{Generalization}}\coloneqq
\text{ACC}_{\text{Train}}-\text{ACC}_{\text{Test}}$. This value allows
us to estimate the amount to which our trained Random Forest
classifier is prone to overfitting. In all our experiments, we
conducted four fold cross-validation, however, we observed stable
results.

\begin{figure}
  \centering
  \includegraphics[width=\textwidth]{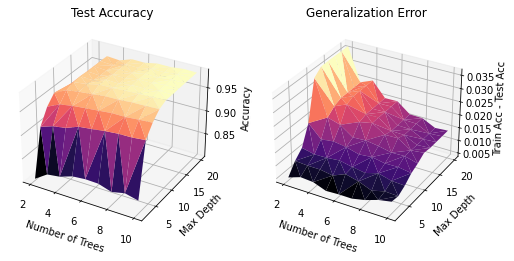}
  \caption{Visualization of distribution of the performance measures accuracy (ACC)
    and generalization error, of trained Random Forest classifiers
    using the hyper parameters \texttt{md} and \texttt{nt}. Mean
    values are reported.}
  \label{fig:ps-acc}
\end{figure}

\subsubsection*{Observations}
From a supervised-learning point of view, we notice that for the model
accuracy the maximum depth \texttt{md} has a greater impact than the
number of trees \texttt{nt}. However, based on the generalization
error (\cref{fig:ps-acc}, right), we find that the number of trees is
instrumental to prevent overfitting. At this point, we feel justified
in stating that very good classifiers exist for \texttt{nt}$\geq8$ and
\texttt{md}$\geq15$. With that, we can turn to the conceptual views
and their capabilities of explaining the Random Forest classifiers. 

For this, we first examine the influence of the parameters on the
number of concepts of each respective conceptual view. In this
experiment, the Random Forest classifier are trained on the entire
data set where no cross-validation was applied. Afterwards, we computed
the different views $\Scon(G, \mathfrak{T})$, also using the
entire data set, and depicted the number of formal concepts per view
and parameter combination in~\cref{fig:ps-scale-sizes}.

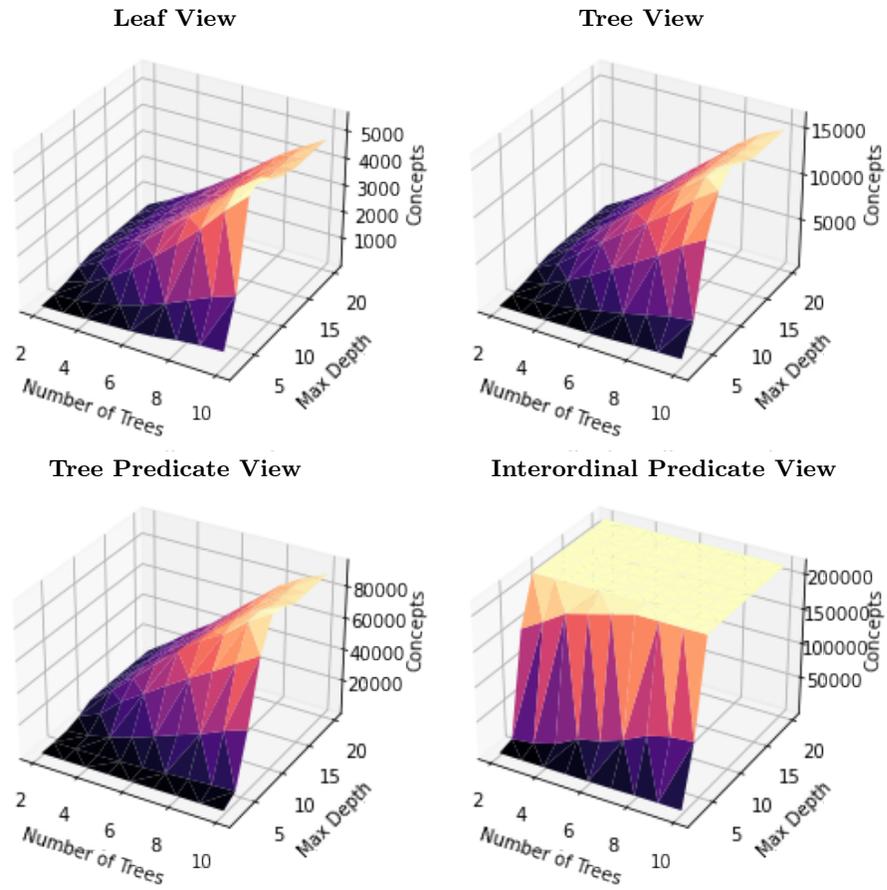
\begin{figure}
  \centering
  \begin{tikzpicture}
    \node[] at (0,0){\includegraphics[width=\textwidth]{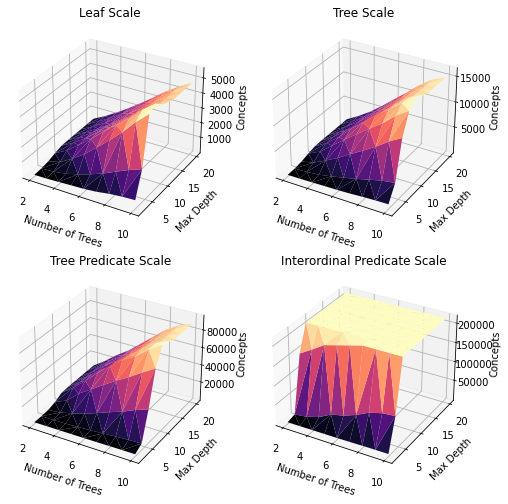}};
    \node[draw=white,fill=white] at (-3.5,5.7){\textbf{Leaf View}};
    \node[draw=white,fill=white] at (2.7,5.7){\textbf{Tree View}};
    \node[draw=white,fill=white] at (-3.5,-0.3){\textbf{Tree Predicate View}};
    \node[draw=white,fill=white] at (3,-0.3){\textbf{Interordinal Predicate View}};
  \end{tikzpicture}
  \caption{Distribution of the number of formal concepts for different
    hyper parameter combinations and different conceptual views: Leaf
    view (top left), tree view (top right), tree predicate view
    (bottom left), and interordinal predicate view (bottom right).}
  \label{fig:ps-scale-sizes}
\end{figure}

\subsubsection*{Observations}

First, we notice that the visual shapes for three different conceptual
views are similar to some extent. The exception is the interordinal
predicate view, which increases more quickly with increasing
\texttt{md}. This observation is expected, since the occurrence of
split predicated increases with the depth of the trees, and, in
contrast to the tree predicate view, the object-predicate incidences
are independent of the location of said predicates in the trees.
For all plots we can report the increasing the depth beyond ten has no
noticeable impact. The performance measurements in \cref{fig:ps-acc}
behaved analogously, yet, since we applied different training sets, we
should refrain from a direct comparison.

In terms of the absolute number of formal concepts, we find that the
leaf view generates the smallest amount (5000), followed by the
tree view (15,000), the tree predicate view (80,000) and the
interordinal predicate view (200,000). This observation is expected
due to our theoretical findings in \cref{dtree-scales}. Obviously, due
to the observed number of formal concepts, all views elude from a
direct human-comprehension. Hence, consecutive data reduction methods,
as proposed in~\cref{sec:post-scaling}, are required.

\subsection{Deriving Meaningful Conceptual Explanations}

\begin{figure}
  \centering
  \colorlet{mivertexcolor}{black!80}
\colorlet{jivertexcolor}{black!80}
\colorlet{vertexcolor}{black!80}
\colorlet{bordercolor}{black!80}
\colorlet{linecolor}{gray}
% parameter corresponds to the used valuation function and can be addressed by #1
\tikzset{vertexbase/.style 2 args={semithick, shape=circle, inner sep=2pt, outer sep=0pt, draw=bordercolor},%
  vertex/.style 2 args={vertexbase={#1}{}, fill=vertexcolor!45},%
  mivertex/.style 2 args={vertexbase={#1}{}, fill=mivertexcolor!45},%
  jivertex/.style 2 args={vertexbase={#1}{}, fill=jivertexcolor!45},%
  divertex/.style 2 args={vertexbase={#1}{}, top color=mivertexcolor!45, bottom color=jivertexcolor!45},%
  conn/.style={-, thick, color=linecolor}%
}
\tikzstyle{a} = [text width=0.8cm,font=\tiny\linespread{-1}\selectfont]
\tikzstyle{r3} = [label distance=4mm]
\tikzstyle{r2} = [label distance=3mm]
\tikzstyle{rr} = [text width=2cm,xshift=0.6cm]
\tikzstyle{ll} = [text width=2cm,xshift=-0.4cm]
\tikzstyle{d} = [text width=2cm, align=center]
\begin{tikzpicture}[scale=0.1,font=\tiny\linespread{0}, label distance = 1mm]
  \begin{scope} %for scaling and the like
    \begin{scope} %draw vertices
      \foreach \nodename/\nodetype/\param/\xpos/\ypos in {%
        0/vertex//81.81818181818183/0.0,
        1/vertex//61.01058710298367/0.29355149181904494,
        2/vertex//38.35418671799808/1.1405197305100927,
        3/vertex//54.34071222329165/23.055822906641012,
        4/vertex//29.143407122232922/30.043310875842167,
        5/vertex//60.48123195380176/32.26660250240617,
        6/vertex//52.01154956689127/33.00769971126084,
        7/vertex//67.57459095283932/40.524542829643906,
        8/vertex//47.141482194417726/41.37151106833496,
        9/vertex//26.390760346487006/42.21847930702601,
        10/vertex//93.40712223291632/44.75938402309916,
        11/vertex//66.5158806544755/50.68816169393651,
        12/vertex//8.075072184793058/55.66410009624642,
        13/vertex//125.59191530317622/55.66410009624642,
        14/vertex//19.826756496631376/56.19345524542833,
        15/vertex//34.75457170356112/56.934552454283,
        16/vertex//110.76997112608285/57.88739172281043,
        17/vertex//96.68912415784413/58.52261790182872,
        18/vertex//54.76419634263717/58.734359961501475,
        19/vertex//86.94898941289708/58.734359961501475,
        20/vertex//72.33878729547646/59.05197305101062,
        21/vertex//59.63426371511071/72.0741097208855,
        22/vertex//16.33301251203079/73.2386910490857,
        23/vertex//106.95861405197311/73.97978825794037,
        24/vertex//50.00000000000001/100.0
      } \node[\nodetype={\param}{}] (\nodename) at (\xpos, \ypos) {};
    \end{scope}
    \begin{scope} %draw connections
      \path (9) edge[conn] (15);
      \path (2) edge[conn] (17);
      \path (0) edge[conn] (18);
      \path (4) edge[conn] (12);
      \path (20) edge[conn] (24);
      \path (0) edge[conn] (13);
      \path (2) edge[conn] (12);
      \path (8) edge[conn] (18);
      \path (22) edge[conn] (24);
      \path (8) edge[conn] (15);
      \path (10) edge[conn] (16);
      \path (7) edge[conn] (18);
      \path (15) edge[conn] (24);
      \path (1) edge[conn] (16);
      \path (14) edge[conn] (24);
      \path (5) edge[conn] (18);
      \path (9) edge[conn] (14);
      \path (10) edge[conn] (17);
      \path (11) edge[conn] (18);
      \path (3) edge[conn] (17);
      \path (5) edge[conn] (17);
      \path (6) edge[conn] (18);
      \path (23) edge[conn] (24);
      \path (13) edge[conn] (24);
      \path (19) edge[conn] (24);
      \path (4) edge[conn] (15);
      \path (7) edge[conn] (19);
      \path (12) edge[conn] (24);
      \path (21) edge[conn] (24);
      \path (11) edge[conn] (20);
      \path (3) edge[conn] (15);
      \path (16) edge[conn] (24);
      \path (18) edge[conn] (24);
      \path (1) edge[conn] (15);
      \path (17) edge[conn] (24);
    \end{scope}
    \begin{scope} %add labels
      \foreach \nodename/\labelpos/\labelopts/\labelcontent in {%
        6/above//{buying $\leq$  med},
        12/above/r3/{buying $\geq$ high},
        13/above/rr/{persons $\leq$  two},
        14/above//{buying $\geq$ med},
        15/above/r2/{maint $\leq$  high},
        16/above//{safety $\leq$  low},
        17/above/r2/{safety $\leq$  med},
        18/above//{buying $\leq$  high},
        19/above//{maint $\geq$ high},
        20/above/r2/{maint $\geq$ med},
        21/above//{persons $\leq$  four},
        22/above//{persons $\geq$ four},
        23/above//{safety $\geq$ med},
        0/below/d/{490, 1520, 1674, 541, 893},
        1/below/d/{300, 1041, 564, 1674, 333},
        2/below/d/{300, 490, 564, 541, 333},
        3/below/d/{300, 564, 1674, 541, 333},
        4/below/d/{188, 300, 564, 541, 333},
        5/below/d/{1344, 490, 564, 1674, 541},
        6/below/ll/{1041,1344,1520, 1674, 893},
        7/below/rr/{1041, 1344, 490, 564, 541, 893},
        8/below/ll/{1041, 564, 1520, 1674, 541},
        9/below/ll/{300, 1041, 564, 541, 333},
        10/below/d/{300, 1344, 564, 1674, 333},
        11/below/d/{1041, 1344, 1520, 541, 893},
        13/below/d/{300},
        14/below/d/{893},
        19/below/d/{188},
        20/below/d/{300},
        21/below/d/{300, 1041, 490, 541, 333},
        22/below/d/{188, 1041, 1344, 564, 333},
        23/below/d/{188, 490, 1520, 541, 893},
        0/right/a/{432P, 0N},
        1/right/a/{432P, 0N},
        2/right/a/{499P, 65N},
        3/right/a/{559P, 193N},
        4/right/a/{468P, 180N},
        5/right/a/{567P, 193N},
        6/right/a/{526P, 338N},
        7/right/a/{458P, 190N},
        8/right/a/{598P, 374N},
        9/right/a/{524P, 302N},
        10/right/a/{464P, 0N},
        11/right/a/{551P, 318N},
        12/right/a/{684P, 180N},
        13/right/a/{576P, 0N},
        14/right/a/{697P, 338N},
        15/right/a/{850P, 446N},
        16/right/a/{576P, 0N},
        17/right/a/{821P, 219N},
        18/right/a/{850P, 446N},
        19/right/a/{674P, 190N},
        20/right/a/{715P, 354N},
        21/right/a/{468P, 264N},
        22/right/a/{634P, 518N},
        23/right/a/{634P, 518N},
        24/right/a/{1210P, 518N}
      } \coordinate[label={[\labelopts]\labelpos:{\labelcontent}}](c) at (\nodename);
    \end{scope}
  \end{scope}
\end{tikzpicture}
  \colorlet{mivertexcolor}{black!80}
\colorlet{jivertexcolor}{black!80}
\colorlet{vertexcolor}{black!80}
\colorlet{bordercolor}{black!80}
\colorlet{linecolor}{gray}
% parameter corresponds to the used valuation function and can be addressed by #1
\tikzset{vertexbase/.style 2 args={semithick, shape=circle, inner sep=2pt, outer sep=0pt, draw=bordercolor},%
  vertex/.style 2 args={vertexbase={#1}{}, fill=vertexcolor!45},%
  mivertex/.style 2 args={vertexbase={#1}{}, fill=mivertexcolor!45},%
  jivertex/.style 2 args={vertexbase={#1}{}, fill=jivertexcolor!45},%
  divertex/.style 2 args={vertexbase={#1}{}, top color=mivertexcolor!45, bottom color=jivertexcolor!45},%
  conn/.style={-, thick, color=linecolor}%
}
\tikzstyle{a} = [text width=0.8cm,font=\tiny\linespread{-1}\selectfont]
\tikzstyle{r3} = [label distance=4mm]
\tikzstyle{r2} = [label distance=3mm]
\tikzstyle{rr} = [text width=2cm,xshift=0.6cm]
\tikzstyle{ll} = [text width=2cm,xshift=-0.4cm]
\tikzstyle{d} = [text width=2cm, align=center]
\begin{tikzpicture}[scale=0.1,font=\tiny\linespread{0}, label distance = 1mm]
  \begin{scope} %for scaling and the like
    \begin{scope} %draw vertices
      \foreach \nodename/\nodetype/\param/\xpos/\ypos in {%
        0/mivertex//-21.817954444610233/0.3519697625149121,
        1/vertex//31.366698748796928/0.7170356111645759,
        2/vertex//81.49770129529665/2.1180493478124873,
        3/vertex//-22.81137421134011/21.87606470832884,
        4/vertex//-3.494878747148121/23.311004371383106,
        5/vertex//77.41364225429604/23.531764319545374,
        6/vertex//21.52069297401347/24.855630413859487,
        7/vertex//18.026948989412894/35.01924927815209,
        8/jivertex//72.0050235243223/42.517119861494,
        9/divertex//86.2440401807838/42.958639757818396,
        10/mivertex//91.76303888483866/51.568277736144005,
        11/vertex//23.743984600577484/55.45235803657366,
        12/mivertex//74.32300298002534/58.85335602549637,
        13/vertex//94.96405813319055/61.281715455280604,
        14/mivertex//-25.46049358928643/62.82703509241587,
        15/vertex//33.261652621857216/65.58653444444329,
        16/mivertex//-0.5004812319538203/66.14533205004815,
        17/mivertex//17.146176406017034/67.02147410749754,
        18/mivertex//86.13366020670273/70.00173340768717,
        19/vertex//56.66987487969204/71.75649663137636,
        20/mivertex//106.55395541170566/76.95567177479631,
        21/mivertex//67.58982456107842/80.7085908935536,
        22/mivertex//45.97690086621753/83.4023099133783,
        23/mivertex//9.87487969201154/85.73147256977869,
        24/vertex//33.333333333333336/100.0
      } \node[\nodetype={\param}{}] (\nodename) at (\xpos, \ypos) {};
    \end{scope}
    \begin{scope} %draw connections
      \path (1) edge[conn] (21);
      \path (22) edge[conn] (24);
      \path (4) edge[conn] (14);
      \path (2) edge[conn] (20);
      \path (9) edge[conn] (23);
      \path (0) edge[conn] (3);
      \path (17) edge[conn] (23);
      \path (15) edge[conn] (23);
      \path (6) edge[conn] (14);
      \path (11) edge[conn] (17);
      \path (2) edge[conn] (16);
      \path (7) edge[conn] (16);
      \path (15) edge[conn] (22);
      \path (19) edge[conn] (22);
      \path (10) edge[conn] (18);
      \path (1) edge[conn] (3);
      \path (5) edge[conn] (8);
      \path (19) edge[conn] (21);
      \path (18) edge[conn] (22);
      \path (5) edge[conn] (9);
      \path (23) edge[conn] (24);
      \path (13) edge[conn] (20);
      \path (7) edge[conn] (15);
      \path (8) edge[conn] (15);
      \path (14) edge[conn] (23);
      \path (6) edge[conn] (15);
      \path (3) edge[conn] (14);
      \path (21) edge[conn] (24);
      \path (16) edge[conn] (23);
      \path (8) edge[conn] (18);
      \path (11) edge[conn] (15);
      \path (4) edge[conn] (17);
      \path (20) edge[conn] (24);
      \path (13) edge[conn] (18);
      \path (3) edge[conn] (16);
      \path (12) edge[conn] (18);
    \end{scope}
    \begin{scope} %add labels
      \foreach \nodename/\labelpos/\labelopts/\labelcontent in {%
        0/above//{buying $\leq$ low, safety $\geq$ med, maint $\geq$ high},
        9/above//{safety $\geq$ high},
        10/above//{buying $\geq$ high},
        12/above//{maint $\leq$  low},
        14/above//{maint $\geq$ med},
        16/above//{buying $\leq$  med},
        17/above//{persons $\leq$  four},
        18/above//{buying $\geq$ med, maint $\leq$  med},
        20/above//{persons $\geq$ more},
        21/above//{safety $\leq$  med},
        22/above//{maint $\leq$  high},
        23/above//{buying $\leq$  high},
        24/above//{persons $\geq$ four, safety $\geq$ med},
        0/below//{1393, 1397, 1444},
        1/below//{1393, 1180, 1444},
        2/below//{1211, 1180, 1397},
        4/below//{686, 1393, 1444},
        5/below//{1211, 686, 851},
        6/below//{686, 1180, 1444},
        7/below//{1211, 1180, 1444},
        8/below//{1180},
        9/below//{1397},
        10/below//{686, 851, 376},
        11/below//{686, 851, 1444},
        12/below//{1211, 851, 376},
        13/below//{1211, 376, 1180},
        19/below//{376, 1180, 1444},
        0/right/a/{4P, 20N},
        1/right/a/{36P, 108N},
        2/right/a/{28P, 164N},
        3/right/a/{42P, 246N},
        4/right/a/{44P, 162N},
        5/right/a/{4P, 92N},
        6/right/a/{42P, 246N},
        7/right/a/{22P, 266N},
        8/right/a/{28P, 164N},
        9/right/a/{35P, 253N},
        10/right/a/{43P, 138N},
        11/right/a/{24P, 192N},
        12/right/a/{21P, 112N},
        13/right/a/{26P, 113N},
        14/right/a/{89P, 318N},
        15/right/a/{58P, 374N},
        16/right/a/{46P, 338N},
        17/right/a/{60P, 228N},
        18/right/a/{47P, 230N},
        19/right/a/{95P, 193N},
        20/right/a/{130P, 254N},
        21/right/a/{165P, 219N},
        22/right/a/{130P, 446N},
        23/right/a/{130P, 446N},
        24/right/a/{250P, 518N}
      } \coordinate[label={[\labelopts]\labelpos:{\labelcontent}}](c) at (\nodename);
    \end{scope}
  \end{scope}
\end{tikzpicture}
  \caption{Tree predicate view scalings for the car data set. The
    centroid elements that have the positive class are displayed top
    and the negative are bottom.}
  \label{fig:ps-view1}
\end{figure}

For our final take on explaining Random Forest classifiers using
conceptual views, we choose an extreme hyper parameters in order to
show the viability of our approach. In detail, we set the number of
trees to 100, which is commonly accepted default
value~\cite{probst2019hyperparameters}. For the maximum depth of the
trees we set no limitation, i.e., the training algorithm splits nodes
until class purity is achieved. The resulting conceptual views, more
precisely their number of formal concepts, naturally rises to the amount
as seen in the last section and more. Although the computation of such
and larger sets of formal concepts is not a challenge for algorithms
from the field of Formal Concept Analysis, human comprehensibility now
requires the application of the selection and aggregation methods
presented in~\cref{sec:post-scaling}.

We want to start with combining the composition method with object,
attribute, and concept selection procedures.  For this, we first
employ \texttt{KMedoids} clustering from sklearn to select a smaller
number of representative objects. We determine the parameter $k$,
i.e., the number of medoids, to be nineteen, by trial and error and
evaluating the silhouette score on the results within the range $2\leq
k\leq 50$. In a second step, we restrict the set of view attributes
(i.e., predicates) in the following way. We computed for all seven
many-valued attributes of the car data table their significance for
the classification using the notion of permutation
importance~\cite{altmann2010permutation}. The result allows us to
select the most important ones. For the rest of our study, we stick to
four. From these data table attributes, we can derive a subset of
important predicates, i.e., view attributes.

Starting from this state, we have applied various other methods
for selection. In \cref{fig:ps-view1} we depict the result for the
tree predicate view when additionally applying
\begin{inparaenum}[a)]
\item composition, more specific, we partition the object set using
  the related class labels, and
\item the TITANIC algorithm~\cite{titanic}.
\end{inparaenum}

The top diagram in~\cref{fig:ps-view1} is comprised of the objects
bearing the positive class label, and the bottom diagram is comprised
of objects bearing the negative class label. The respective values for
minimum support are five and three, i.e., all concepts in the view
using the positive labels have an extent size of at least five, and
analogously three for the negative part.  These values were chosen
such that the resulting iceberg concept lattices is of comprehensible
size. The particular values three and five seem to reflect the
imbalance of the class labels to some extent, however, this
observation is not essential.  Both diagrams are annotated in the
usual way. In addition to that, we annotated on the right to each
concept node the class purity in this concept, i.e., the number of
positive and negative labeled objects of the data table.

% | scale | class | min-cnt | concepts |
% |-------+-------+---------+----------|
% | inter | p     |       6 |       26 |
% | tree  | p     |       5 |       25 |
% | inter | n     |       3 |       25 |
% | tree  | n     |       3 |       25 |

First of all, we observe structural differences between the iceberg
concept lattices of the positive (PICL) and negative (NICL) center
objects, although both have twenty-five formal concepts. PICL has
twelve co-atoms while NICL has four co-atoms. NICL has a longest chain
of five elements while PICL's is three. We claim that PICL is easier
to comprehend than NICL due to its smaller depth. At the same time,
NICL implies that the description of the negative class is more
difficult and demanding with respect to the number of attributes,
i.e., intent sizes.  More generally, in both diagrams we can infer
descriptions of the positive and negative class from the concepts
lowest in the diagrams. Even though there are methods to explain the
influence of single attributes on the classification, the iceberg concept
lattice allows to easily comprehend the influence of arbitrary
attribute combinations. For example, the concept with extent label
\texttt{300, 490, 564, 541, 333} is a result of the attribute
combination $\emph{buying}\geq\emph{high}$ and
$\emph{safety}\geq\emph{med}$.  Furthermore, the conceptual structures
allows to identify attributes with a high global influence on the
classification. For example, $\emph{maint}\leq\emph{high}$ has five
direct lower neighbors whereas $\emph{persons}\leq\emph{two}$ has one
direct lower neighbor. 

A particularly interesting observation in the NICL diagram is the
presence of attributes that support all objects. Hence, these are
essential for classification of all objects with the negative class
label. This conclusion is especially easy to infer from the conceptual
structure compared with analyzing all hundred trees of the underlying
Random Forest. Finally, the iceberg concept lattice of NICL reveals
redundant attributes. For example, the concept annotated with the
object extent \texttt{1393, 1397, 1444} has three annotated attributes
of which only one is needed to identify this concept.

A more general inference about the Random Forest that the tree
predicate view allows is to identify ``costly'' objects. By this we
mean objects whose classification required a large number of
(potentially redundant) threshold value tests. For example, we refer
the reader to the concept bearing the objects \texttt{1393,1397,1444}
within NICL. On the one hand, all these objects required redundant
testing of attributes, namely \emph{maint$\geq$med} and
\emph{maint$\geq$high}. On the other hand, the composition of the
intent includes four data table attributes, i.e.,
\emph{buying,safety,maint}, and \emph{persons}.

\begin{figure}
  \centering
  \colorlet{mivertexcolor}{black!80}
\colorlet{jivertexcolor}{black!80}
\colorlet{vertexcolor}{black!80}
\colorlet{bordercolor}{black!80}
\colorlet{linecolor}{gray}
% parameter corresponds to the used valuation function and can be addressed by #1
\tikzset{vertexbase/.style 2 args={semithick, shape=circle, inner sep=2pt, outer sep=0pt, draw=bordercolor},%
  vertex/.style 2 args={vertexbase={#1}{}, fill=vertexcolor!45},%
  mivertex/.style 2 args={vertexbase={#1}{}, fill=mivertexcolor!45},%
  jivertex/.style 2 args={vertexbase={#1}{}, fill=jivertexcolor!45},%
  divertex/.style 2 args={vertexbase={#1}{}, top color=mivertexcolor!45, bottom color=jivertexcolor!45},%
  conn/.style={-, thick, color=linecolor}%
}
\tikzstyle{a} = [text width=0.8cm,font=\tiny\linespread{-1}\selectfont]
\tikzstyle{r3} = [label distance=4mm]
\tikzstyle{r2} = [label distance=3mm]
\tikzstyle{rr} = [text width=2cm,xshift=0.6cm]
\tikzstyle{ll} = [text width=2cm,xshift=-0.4cm]
\tikzstyle{d} = [text width=2cm, align=center]
\begin{tikzpicture}[scale=0.15,font=\tiny\linespread{0}, label distance = 1mm]
  \begin{scope} %for scaling and the like
    \begin{scope} %draw vertices
      \foreach \nodename/\nodetype/\param/\xpos/\ypos in {%
        0/vertex//43.75360923965352/20.514918190567855,
        1/mivertex//13.686236766121262/22.420596727622723,
        2/vertex//43.118383060635246/32.90182868142446,
        3/jivertex//13.262752646775738/34.27815206929742,
        4/jivertex//56.14051973051013/40.630413859480285,
        5/vertex//13.051010587102972/42.00673724735324,
        6/jivertex//43.118383060635246/47.61790182868145,
        7/jivertex//33.80173243503369/52.008796920115499,
        8/jivertex//24.37921077959576/52.69971126082775,
        9/mivertex//79.53801732435039/54.60538979788261,
        10/vertex//60.58710298363814/54.011260827718995,
        11/mivertex//43.43599615014438/54.92300288739175,
        12/jivertex//54.0230991337825/55.13474494706452,
        13/mivertex//13.79210779595764/58.84023099133786,
        14/vertex//6.651588065447534/63.8517805582291,
        15/vertex//21.944177093358995/61.69874879692015,
        16/mivertex//-2.194417709335916/68.89797882579404,
        17/mivertex//70.22136669874882/69.00384985563045,
        18/mivertex//14.850818094321454/69.21559191530322,
        19/mivertex//55.50529355149184/69.21559191530322,
        20/mivertex//43.330125120308004/69.63907603464875,
        21/mivertex//31.049085659287776/69.74494706448513,
        22/vertex//42.059672762271425/92.71896053897984
      } \node[\nodetype={\param}{}] (\nodename) at (\xpos, \ypos) {};
    \end{scope}
    \begin{scope} %draw connections
      \path (14) edge[conn] (18);
      \path (19) edge[conn] (22);
      \path (14) edge[conn] (16);
      \path (2) edge[conn] (6);
      \path (7) edge[conn] (21);
      \path (5) edge[conn] (21);
      \path (5) edge[conn] (16);
      \path (4) edge[conn] (20);
      \path (16) edge[conn] (22);
      \path (10) edge[conn] (19);
      \path (7) edge[conn] (20);
      \path (12) edge[conn] (20);
      \path (21) edge[conn] (22);
      \path (15) edge[conn] (18);
      \path (8) edge[conn] (20);
      \path (11) edge[conn] (20);
      \path (4) edge[conn] (17);
      \path (10) edge[conn] (17);
      \path (6) edge[conn] (21);
      \path (8) edge[conn] (18);
      \path (9) edge[conn] (17);
      \path (17) edge[conn] (22);
      \path (13) edge[conn] (18);
      \path (3) edge[conn] (20);
      \path (20) edge[conn] (22);
      \path (18) edge[conn] (22);
      \path (0) edge[conn] (4);
      \path (6) edge[conn] (19);
      \path (12) edge[conn] (19);
      \path (2) edge[conn] (7);
      \path (1) edge[conn] (3);
      \path (2) edge[conn] (12);
      \path (3) edge[conn] (16);
      \path (15) edge[conn] (21);
      \path (0) edge[conn] (8);
    \end{scope}
    \begin{scope} %add labels
      \foreach \nodename/\labelpos/\labelopts/\labelcontent in {%
        1/above//{persons $\geq$ more},
        9/above//{safety $\geq$ high},
        11/above//{maint $\geq$ vhigh},
        13/above//{buying $\geq$ vhigh},
        16/above//{persons $\geq$ four},
        17/above//{safety $\geq$ med},
        18/above//{buying $\geq$ high},
        19/above//{maint $\geq$ med},
        20/above//{maint $\geq$ high},
        21/above//{buying $\geq$ med},
        0/below//{188, 490, 541},
        1/below//{188, 1344, 564},
        2/below//{1041, 541, 893},
        3/below//{1041},
        4/below//{893},
        5/below//{1041, 564, 333},
        6/below//{300},
        7/below//{564},
        8/below//{564},
        9/below//{188, 1520, 893},
        10/below//{1520, 541, 893},
        11/below//{1344, 490, 893},
        12/below//{1344},
        13/below//{188, 300, 333},
        14/below//{188, 564, 333},
        15/below//{300, 564, 541, 333},
        22/below//{1674},
        0/right//{252P, 36N},
        1/right//{192P, 94N},
        2/right//{287P, 108N},
        3/right//{386P, 190N},
        4/right//{386P, 190N},
        5/right//{416P, 338N},
        6/right//{464P, 226N},
        7/right//{351P, 108N},
        8/right//{396P, 36N},
        9/right//{277P, 299N},
        10/right//{389P, 354N},
        11/right//{360P, 72N},
        12/right//{452P, 190N},
        13/right//{360P, 72N},
        14/right//{396P, 180N},
        15/right//{447P, 174N},
        16/right//{634P, 518N},
        17/right//{634P, 518N},
        18/right//{684P, 180N},
        19/right//{715P, 354N},
        20/right//{674P, 190N},
        21/right//{697P, 338N},
        22/right//{1210P, 518N}
      } \coordinate[label={[\labelopts]\labelpos:{\labelcontent}}](c) at (\nodename);
    \end{scope}
  \end{scope}
\end{tikzpicture}
  \colorlet{mivertexcolor}{black!80}
\colorlet{jivertexcolor}{black!80}
\colorlet{vertexcolor}{black!80}
\colorlet{bordercolor}{black!80}
\colorlet{linecolor}{gray}
% parameter corresponds to the used valuation function and can be addressed by #1
\tikzset{vertexbase/.style 2 args={semithick, shape=circle, inner sep=2pt, outer sep=0pt, draw=bordercolor},%
  vertex/.style 2 args={vertexbase={#1}{}, fill=vertexcolor!45},%
  mivertex/.style 2 args={vertexbase={#1}{}, fill=mivertexcolor!45},%
  jivertex/.style 2 args={vertexbase={#1}{}, fill=jivertexcolor!45},%
  divertex/.style 2 args={vertexbase={#1}{}, top color=mivertexcolor!45, bottom color=jivertexcolor!45},%
  conn/.style={-, thick, color=linecolor}%
}
\tikzstyle{a} = [text width=0.8cm,font=\tiny\linespread{-1}\selectfont]
\tikzstyle{r3} = [label distance=4mm]
\tikzstyle{r2} = [label distance=3mm]
\tikzstyle{rr} = [text width=2cm,xshift=0.6cm]
\tikzstyle{ll} = [text width=2cm,xshift=-0.4cm]
\tikzstyle{d} = [text width=2cm, align=center]
\begin{tikzpicture}[scale=0.1,font=\tiny\linespread{0}, label distance = 1mm]
  \begin{scope} %for scaling and the like
    \begin{scope} %draw vertices
      \foreach \nodename/\nodetype/\param/\xpos/\ypos in {%
        0/vertex//72.23291626564009/31.631376323387883,
        1/vertex//49.153031761308974/32.16073147256979,
        2/vertex//3.8402309913378048/35.65447545717038,
        3/vertex//39.09528392685275/40.31280076997115,
        4/vertex//8.392685274302202/41.37151106833496,
        5/divertex//96.15976900866224/47.40615976900869,
        6/jivertex//-7.699711260827748/47.51203079884507,
        7/vertex//7.22810394610201/48.04138594802698,
        8/vertex//41.106833493744/55.66410009624642,
        9/vertex//72.55052935514922/56.828681424446614,
        10/vertex//86.94898941289708/57.56977863330128,
        11/vertex//57.30510105871033/57.99326275264681,
        12/vertex//25.649663137632338/58.09913378248319,
        13/mivertex//-6.111645813282024/60.64003849855634,
        14/divertex//96.79499518768051/62.86333012512035,
        15/mivertex//28.190567853705495/71.65062560153999,
        16/mivertex//71.59769008662178/73.29740134744952,
        17/mivertex//57.093358999037555/74.4032723772859,
        18/mivertex//88.21944177093366/75.46198267564971,
        19/vertex//71.80943214629454/94.73051010587109
      } \node[\nodetype={\param}{}] (\nodename) at (\xpos, \ypos) {};
    \end{scope}
    \begin{scope} %draw connections
      \path (7) edge[conn] (15);
      \path (5) edge[conn] (14);
      \path (16) edge[conn] (19);
      \path (13) edge[conn] (15);
      \path (6) edge[conn] (13);
      \path (6) edge[conn] (12);
      \path (3) edge[conn] (8);
      \path (0) edge[conn] (10);
      \path (18) edge[conn] (19);
      \path (2) edge[conn] (7);
      \path (17) edge[conn] (19);
      \path (10) edge[conn] (16);
      \path (11) edge[conn] (16);
      \path (11) edge[conn] (17);
      \path (7) edge[conn] (16);
      \path (2) edge[conn] (11);
      \path (0) edge[conn] (11);
      \path (1) edge[conn] (7);
      \path (1) edge[conn] (10);
      \path (3) edge[conn] (12);
      \path (1) edge[conn] (8);
      \path (4) edge[conn] (8);
      \path (14) edge[conn] (16);
      \path (9) edge[conn] (18);
      \path (15) edge[conn] (19);
      \path (12) edge[conn] (15);
      \path (0) edge[conn] (5);
      \path (8) edge[conn] (15);
      \path (0) edge[conn] (9);
      \path (4) edge[conn] (13);
      \path (9) edge[conn] (17);
      \path (12) edge[conn] (17);
      \path (10) edge[conn] (18);
      \path (8) edge[conn] (18);
      \path (3) edge[conn] (9);
      \path (2) edge[conn] (6);
    \end{scope}
    \begin{scope} %add labels
      \foreach \nodename/\labelpos/\labelopts/\labelcontent in {%
        4/above//{buying $\geq$ vhigh},
        5/above//{maint $\geq$ vhigh},
        13/above//{buying $\geq$ high},
        14/above//{maint $\geq$ high},
        15/above//{buying $\geq$ med},
        16/above//{maint $\geq$ med},
        17/above//{safety $\geq$ high},
        18/above//{persons $\geq$ more},
        19/above//{persons $\geq$ four, safety $\geq$ med},
        0/below//{1397},
        1/below//{1180},
        2/below//{686},
        3/below//{1211},
        4/below//{376},
        5/below//{1393},
        6/below//{851},
        14/below//{1444},
        0/right//{26P, 22N},
        1/right//{104P, 112N},
        2/right//{75P, 69N},
        3/right//{45P, 99N},
        4/right//{60P, 36N},
        5/right//{120P, 72N},
        6/right//{77P, 115N},
        7/right//{206P, 226N},
        8/right//{118P, 170N},
        9/right//{49P, 143N},
        10/right//{114P, 174N},
        11/right//{81P, 207N},
        12/right//{81P, 207N},
        13/right//{204P, 180N},
        14/right//{194P, 190N},
        15/right//{232P, 344N},
        16/right//{222P, 354N},
        17/right//{85P, 299N},
        18/right//{130P, 254N},
        19/right//{250P, 518N}
      } \coordinate[label={[\labelopts]\labelpos:{\labelcontent}}](c) at (\nodename);
    \end{scope}
  \end{scope}
\end{tikzpicture}

  \caption{Forest view based on individual tree predicate views. We
    restricted the predicates to the ones using expressions with
    $\geq$. The diagram at the top shows the related iceberg concept
    lattice for the center objects bearing the positive class label,
    whereas the bottom diagram shows the analogue for the negative
    class labels.}
  \label{fig:ps-view2}
\end{figure}
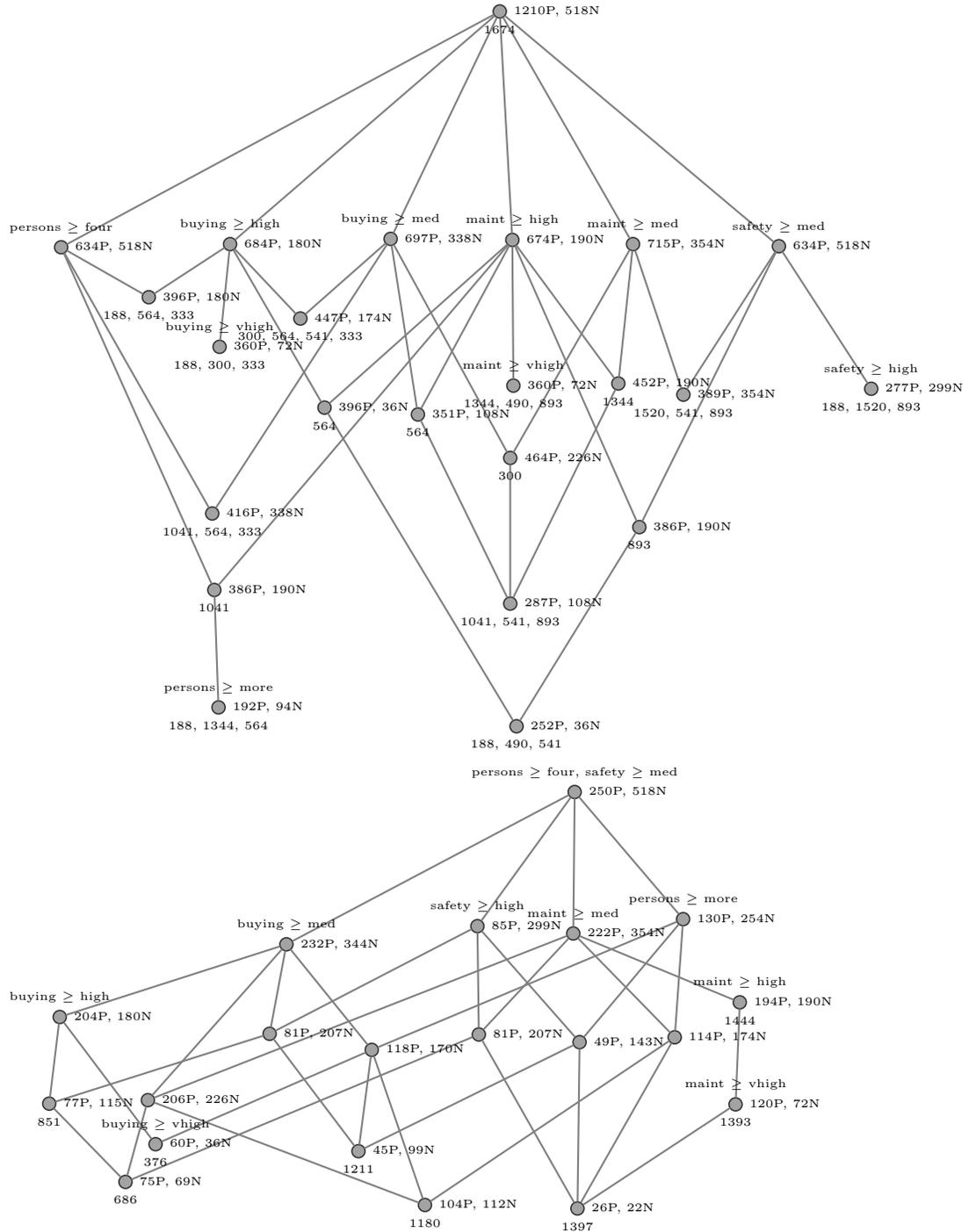

\subsubsection*{Ordinal Factors}
Interordinal scalings are, in general, more complex for human
readers. The reason for this is that in interordinal scaling an
interval of threshold values has to be considered instead of only a
single one. For example, in~\cref{fig:ps-view1} (bottom) we find the
concept $c$ with the extent \texttt{686,1180,1444} that is a lower
neighbor to two concepts bearing the attributes \emph{maint$\geq$med}
and \emph{maint$\leq$high} respectively. Thus, the human reader has to
consider the interval $[\text{\emph{med,high}}]$ within the linear
order of threshold values for \emph{maint}, which is
\emph{low,med,high,vhigh}. Moreover, $c$ has the attributes
\emph{buying$\leq$high}, \emph{persons$\geq$four}, and
\emph{safety$\geq$med}. Thus, a human reader has to comprehend
different \emph{directions} of order, i.e., $\geq$ and $\leq$, at the same
time. With this in mind, we focused on the $\geq$-ordering
in~\cref{fig:ps-view2}. Of course, this limits the expressiveness of
possible explanations based on the views. However, as illustrated
above, the comprehensibility increases. Furthermore, this approach
results in fewer concepts in general. Hence, it allows us to use lower
support values for the iceberg concept lattice. We present
in~\cref{fig:ps-view2} the iceberg concept lattice using the support
values of one for the negative class and three for the positive class.
This approach altogether can be considered as an ordinal factor
approach with respect to the context operation
\emph{apposition}~\cite{fca-book}.

In addition to the advantages discussed above, one may apply all
analysis deductions to the diagrams in~\cref{fig:ps-view2} that were
explained for~\cref{fig:ps-view1}.

\subsubsection{The Interordinal Predicate View}
As for our last analysis example we present an ordinal factor of the
interordinal predicate view on the Random Forest. As
in~\cref{fig:ps-view2}, we choose $\geq$ and the same parameter for
support. In contrast to the examples based on the tree predicate view,
the lattices shown in~\cref{fig:ps-view3} encode a different kind of
information for explaining the Random Forest. More precisely, the
interordinal predicate view represents the model relationship between
objects and the predicates of the Random Forest.

A distinctive feature of the interordinal predicate view is that it
reflects implications between attribute thresholds values that are
enforced by their order relation. For example, we find
in~\cref{fig:ps-view3} that
\emph{vhigh}$\to$\emph{high}$\to$\emph{med} for the \emph{buying}
attribute of the data table. Moreover, one can easily read the
corresponding chains from the diagram.

Furthermore, one can infer from interordinal predicate view all valid
attribute implications between values of different attributes of the
data table. For example, one can find in~\cref{fig:ps-view3} that the
attribute value \emph{safety$\geq$med} implies
\emph{maint$\geq$med}. Although there are also  implications present
in the iceberg concept lattice of the tree predicate view, we may
note that those are not necessarily implications within the data
table.

We would like to conclude the analysis of the interordinal predicate
view by emphasizing two important facts. First, the particular
attribute threshold values were derived by the training procedure
(i.e., Random Forest) and do therefore represent the ``view'' of the
trained classifier function on the data. Hence, when revealing
threshold value implications by means of the interordinal predicate
view, we actually find implications that are valid within the data
table when viewed through the scaling of the Random Forest. Second,
the set of all valid implications with respect to all data objects
bearing the same class label is the implicational theory of this class
as ``seen'' by the Random Forest. Thus, by computing both
implicational theories, i.e., for both class labels, one can compare
both theories for similarities and differences.

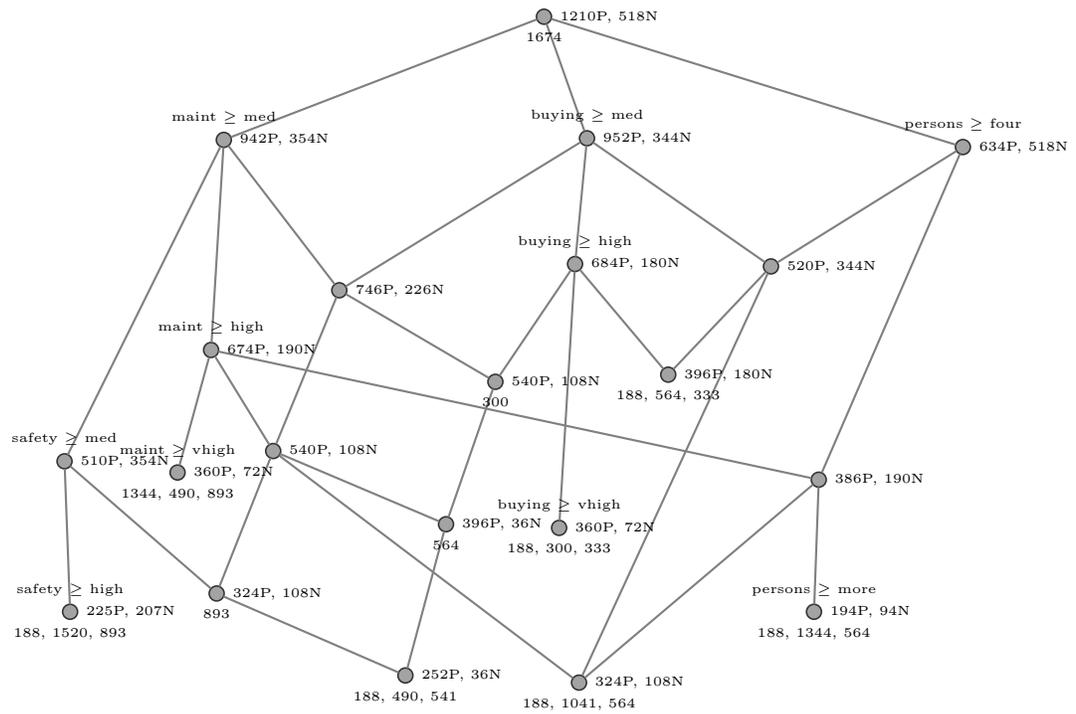
\begin{figure}
  \centering

  \colorlet{mivertexcolor}{black!80}
\colorlet{jivertexcolor}{black!80}
\colorlet{vertexcolor}{black!80}
\colorlet{bordercolor}{black!80}
\colorlet{linecolor}{gray}
% parameter corresponds to the used valuation function and can be addressed by #1
\tikzset{vertexbase/.style 2 args={semithick, shape=circle, inner sep=2pt, outer sep=0pt, draw=bordercolor},%
  vertex/.style 2 args={vertexbase={#1}{}, fill=vertexcolor!45},%
  mivertex/.style 2 args={vertexbase={#1}{}, fill=mivertexcolor!45},%
  jivertex/.style 2 args={vertexbase={#1}{}, fill=jivertexcolor!45},%
  divertex/.style 2 args={vertexbase={#1}{}, top color=mivertexcolor!45, bottom color=jivertexcolor!45},%
  conn/.style={-, thick, color=linecolor}%
}
\tikzstyle{a} = [text width=0.8cm,font=\tiny\linespread{-1}\selectfont]
\tikzstyle{r3} = [label distance=4mm]
\tikzstyle{r2} = [label distance=3mm]
\tikzstyle{rr} = [text width=2cm,xshift=0.6cm]
\tikzstyle{ll} = [text width=2cm,xshift=-0.4cm]
\tikzstyle{d} = [text width=2cm, align=center]
\begin{tikzpicture}[scale=0.1,font=\tiny\linespread{0}, label distance = 1mm]
  \begin{scope} %for scaling and the like
    \begin{scope} %draw vertices
      \foreach \nodename/\nodetype/\param/\xpos/\ypos in {%
        0/vertex//54.65832531280079/11.410009624639073,
        1/vertex//31.578440808469686/12.3628488931665,
        2/mivertex//-12.993262752646807/20.832531280077006,
        3/mivertex//85.89027911453326/20.832531280077006,
        4/jivertex//6.4870067372473414/23.26756496631377,
        5/mivertex//52.01154956689127/31.996150144369594,
        6/jivertex//36.977863330125125/32.47834456207893,
        7/vertex//86.52550529355156/38.40712223291628,
        8/mivertex//1.28873917228106194/39.35996150144372,
        9/mivertex//-13.734359961501475/40.853897978825816,
        10/vertex//14.003849855630406/42.21847930702601,
        11/jivertex//43.54186717998076/51.429258902791176,
        12/vertex//66.5158806544755/52.3820981713186,
        13/mivertex//5.745909528392673/55.66410009624642,
        14/vertex//22.791145332050043/63.60442733397502,
        15/vertex//80.17324350336867/66.78055822906646,
        16/mivertex//54.12897016361888/67.0981713185756,
        17/mivertex//105.68816169393654/82.66121270452363,
        18/mivertex//7.439846005774768/83.61405197305106,
        19/mivertex//55.71703561116461/83.82579403272382,
        20/vertex//50.0/100.0
      } \node[\nodetype={\param}{}] (\nodename) at (\xpos, \ypos) {};
    \end{scope}
    \begin{scope} %draw connections
      \path (17) edge[conn] (20);
      \path (16) edge[conn] (19);
      \path (7) edge[conn] (13);
      \path (14) edge[conn] (19);
      \path (13) edge[conn] (18);
      \path (0) edge[conn] (15);
      \path (10) edge[conn] (14);
      \path (1) edge[conn] (4);
      \path (0) edge[conn] (10);
      \path (12) edge[conn] (15);
      \path (19) edge[conn] (20);
      \path (15) edge[conn] (17);
      \path (11) edge[conn] (14);
      \path (1) edge[conn] (6);
      \path (6) edge[conn] (10);
      \path (18) edge[conn] (20);
      \path (4) edge[conn] (9);
      \path (5) edge[conn] (16);
      \path (6) edge[conn] (11);
      \path (15) edge[conn] (19);
      \path (11) edge[conn] (16);
      \path (12) edge[conn] (16);
      \path (10) edge[conn] (13);
      \path (4) edge[conn] (10);
      \path (0) edge[conn] (7);
      \path (7) edge[conn] (17);
      \path (3) edge[conn] (7);
      \path (9) edge[conn] (18);
      \path (8) edge[conn] (13);
      \path (14) edge[conn] (18);
      \path (2) edge[conn] (9);
    \end{scope}
    \begin{scope} %add labels
      \foreach \nodename/\labelpos/\labelopts/\labelcontent in {%
        2/above//{safety $\geq$ high},
        3/above//{persons $\geq$ more},
        5/above//{buying $\geq$ vhigh},
        8/above//{maint $\geq$ vhigh},
        9/above//{safety $\geq$ med},
        13/above//{maint $\geq$ high},
        16/above//{buying $\geq$ high},
        17/above//{persons $\geq$ four},
        18/above//{maint $\geq$ med},
        19/above//{buying $\geq$ med},
        0/below//{188, 1041, 564},
        1/below//{188, 490, 541},
        2/below//{188, 1520, 893},
        3/below//{188, 1344, 564},
        4/below//{893},
        5/below//{188, 300, 333},
        6/below//{564},
        8/below//{1344, 490, 893},
        11/below//{300},
        12/below//{188, 564, 333},
        20/below//{1674},
        0/right//{324P, 108N},
        1/right//{252P, 36N},
        2/right//{225P, 207N},
        3/right//{194P, 94N},
        4/right//{324P, 108N},
        5/right//{360P, 72N},
        6/right//{396P, 36N},
        7/right//{386P, 190N},
        8/right//{360P, 72N},
        9/right//{510P, 354N},
        10/right//{540P, 108N},
        11/right//{540P, 108N},
        12/right//{396P, 180N},
        13/right//{674P, 190N},
        14/right//{746P, 226N},
        15/right//{520P, 344N},
        16/right//{684P, 180N},
        17/right//{634P, 518N},
        18/right//{942P, 354N},
        19/right//{952P, 344N},
        20/right//{1210P, 518N}
      } \coordinate[label={[\labelopts]\labelpos:{\labelcontent}}](c) at (\nodename);
    \end{scope}
  \end{scope}
\end{tikzpicture}
  \caption{Forest view based on individual interodinal predicate
    views. We restricted the predicates to the ones using expressions
    with $\geq$. The diagram at shows the related iceberg concept
    lattice for the center objects bearing the positive class label. }
  \label{fig:ps-view3}
\end{figure}

% - only $geq$
% | scale | class | min-cnt | concepts |
% |-------+-------+---------+----------|
% | inter | p     |       3 |       21 |
% | tree  | p     |       3 |       23 |
% | inter | n     |       1 |       20 |
% | tree  | n     |       1 |       22 |

\paragraph{Discussion and Outlook}
Obviously, our approach is capable of identifying important
combinations of attribute threshold values and their influence on the
classification results. Certainly, there is a wide range of
combinatorial methods to identify interesting and meaningful
combinations. However, the major advantage of the proposed conceptual
method is that it provides a structured mathematical way to directly
and efficiently identify the important combinations and, at the same
time, their semantic interpretation~\cite{fca-book}.

At this point, we would like to conclude our study by pointing out
that the developed conceptual structures allow for the possibility
for the application of a variety of other conceptual methods. For
example, as outlined in earlier in this section, an analysis of implication
structures between attribute threshold values within a class can
reveal new insights into a Random Forest. Likewise, the conceptual
views allow to compare different trained Random Forest classifiers for
their implicational differences and similarities. A detailed
investigation of these questions is planned as future work.

\section{Related Work}
There are a multitude of classification methods using trees and
tree-ensembles, e.g., decision tree~\cite{dtree}, Random
Forest\cite{rforest}, or decision stumps~\cite{iba1992induction}, to
name a few. The most important property of a single decision tree
classifier is its human interpretability. For example, the
visualization of such a tree provides insights to the classification
process and, at the same time, presents a scaled view on the data
set. Unfortunately, the latter approach received only very little
research attention, so far. Methods that address these \emph{scalings}
are \texttt{RandomTreesEmbedding}, as implemented in \emph{sklearn},
and \emph{tree views}~\cite{dudyrev}. The first method extracts a
partition of the data set objects depending on the tree leafs that
classify them. This partition view, however, is a very coarse scaling
of the data set and makes very little use of the hierarchical tree
structures. The second method analyses the order structure of the
trees through a concept lattice. While the authors provide a novel
translation of the tree structures into the realm of Formal Concept
Analysis, they do not elaborate how those can be utilized for
interpretation. Furthermore, they solely reflect the order structure
and its hierarchy induced on the data set. Proceeding in this manner
does not account for the used tree predicates, which are essential for
human comprehension. Nonetheless, the translation approach itself is
fruitful since this enables the application of FCA based post
processing methods, thus explanations. For example, scale-measures
\cite{smeasure}, TITANIC~\cite{titanic}, core
structures~\cite{pqcores} or importance measures~\cite{measures}.

Other methods that try to achieve a unified view on tree ensembles
combine all trees into a new tree through
merging~\cite{Tmerging}. Yet, there are two main disadvantages of
these approaches. The first is that their output is again a tree,
which in contrast to a lattice order allows only for linear paths for
each node. Thus, they loose the ability to cope with missing
information and do not cover the concurrency of tree ensemble. The
second disadvantage regards the interpretation of the output: the sole
goal of the outputted merged tree is to induce the same partition on
the data set, as the ensemble would. This, however, omits the internal
representations of the trees.

A different, yet related line of research is the construction
classifiers from concept
lattices~\cite{prokasheva2013classification,DBLP:journals/ijufks/BelohlavekBOV08}.
For example, one may compute the concept lattice in a top-to-bottom
fashion with class purity, used as stopping
criterion~\cite{conf/cla/BelohlavekBOV07}, and then select a tree from
the generated partial ordered set. Although mathematically elegant,
these approaches are outperformed by methods like Random Forest.

\section{Conclusion}
We have formally introduced conceptual views on trees and tree
ensemble classifiers, a novel approach for analyzing and globally
explaining tree based machine learning models. In order to achieve
this, we have interpreted the splitting predicates of decision trees
as attributes of a formal context and thus transformed them into a
conceptual structure through conceptual scaling. Equipped with this
method, a user can gain new in-depth knowledge about tree-based
classification functions. At least as important is the now possible
analysis of the view of a classifier on its training data as well as
on previously unknown data.

In detail we introduced two novel approaches, i.e., views, for the
conceptualization of tree ensembles and compared them previous
work. To this end, we proved that the methods we presented have higher
granularity, and thus expressive power, for explaining tree-based
classifiers. Our main theoretical result is that any tree ensemble is
embedded in the introduced tree conceptual views. In order to underpin
our theoretical modeling and results, we conducted a parameter study
on how the different views are effected by parameters of the
classification models. Here we compared the formal concepts with
respect to the set of decision trees. Since the resulting conceptual
views may become large, we introduced different scaling and
post-processing methods that preserve significant parts of the
relevant knowledge within a tree ensemble. The first method is logical
scaling, which allows for comprehending the classification of a single
object, and how the classification of said object would change with
respect to changes in the attribute values.. A second scaling method
is in regard to the number of thresholds induced by the node
predicates for each attribute. We argued how those can be reduced
using conceptual views. Overall the scaling of conceptual views of
tree ensemble classifiers should be further investigated in future
work. Finally we want to point out that we demonstrated the
applicability of our approach on a real world data set and a Random
Forest trained with standard parameters, i.e., 100 trees and no depth
limitation. In particular, the number of trees used in the evaluation
distinguishes our method from that of previous work based on FCA,
which focused on a small single-digit number.

In our work, we did not focus on technical details of the tree
ensembles, especially hyperparameter studies for computing trees and
Random Forests, since our approach empahsizes explaining a
\emph{given} forest with respect to known, and potentially unknown,
data. Nonetheless, a detailed study investigating the relationship
between the hyperparameters and the resulting forests and their
different conceptual views could provide deeper insights into the
training process of Random Forests. Another closely related topic we
did not dive into is that the introduced views in combination with the
classifier can be employed for automatically scaling of many-valued
data tables. Furthermore, we could also envision applications for
enumerating distinct decision trees, as they are lattice
ordered~\cite{pmlr-v70-ruggieri17a}. Also, surrogate-based
approaches~\cite{pmlr-v151-kim22b} for explaining black-box classifier
functions may profit from in-depth explanations based on our
conceptual views. The same applies to bandit-based
approaches~\cite{2207.06355}. In conclusion, our work has opened a
door for the conceptual analysis of diverse supervised
machine-learning models, which should be walked with future work on
explainable AI.

\printbibliography
\end{document}